%% file: arxiv_pref_comp.tex
\documentclass{article} 

\usepackage{graphicx,epstopdf,amsfonts,amsmath,algorithmic,hyperref, enumerate, paralist, bm,mathabx}
\usepackage{amsthm}
\usepackage{latexsym}
\usepackage{amsmath,amssymb,amsfonts}
\usepackage{url}
\usepackage{enumitem}

\usepackage{lmodern}

\usepackage[affil-it]{authblk}

\usepackage[round]{natbib}
\usepackage{bm}
\usepackage{algorithmic}
\usepackage{algorithm}
\usepackage{wrapfig}

\usepackage{geometry}

\newcommand{\sX}{{\mathcal X}} 
\newcommand{\sY}{{\mathcal Y}}

\newcommand{\sG}{\mathcal{G}}
\newcommand{\sP}{\mathcal{P}}
\newcommand{\sS}{\mathcal{S}}
\newcommand{\sO}{\mathcal{O}}

\newcommand{\sM}{\mathcal{M}}

\newcommand{\ba}{\bm{a}}

\newcommand{\bg}{\bm{g}}

\newcommand{\bu}{\bm{u}}
\newcommand{\bv}{\bm{v}}

\newcommand{\bx}{\bm{x}}
\newcommand{\by}{\bm{y}}
\newcommand{\bz}{\bm{z}}

\newcommand{\bOmega}{\bm{\Omega}}

\newcommand{\bbE}{\mathbb{E}} 
\newcommand{\bbR}{\mathbb{R}} 
\newcommand{\bbN}{\mathbb{N}} 

\newcommand{\argmin}{\operatornamewithlimits{arg\ min}}

\newcommand{\dis}{\mathop { \textnormal{dis}}}
\newcommand{\edis}{\mathop { \widehat{\textnormal{dis}}}}

\newcommand{\ind}[1]{\bm{1}{\{#1\}}} 

\newcommand{\set}[1]{\{#1\}}

\newcommand{\norm}[1]{\left\|#1\right\|}

\renewcommand{\set}[1]{\{#1\}}

\newtheorem{thm}{Theorem}
\newtheorem{prop}{Proposition}
\newtheorem{lemma}{Lemma}
\newtheorem{cor}{Corollary}
\newtheorem{defn}{Definition}

\newcommand{\nrho}{\mathop{\rho^\prime}}
\newcommand{\diverse}{discerning}
\newcommand{\discern}{discriminative}

\newcommand{\local}{consistent}
\newcommand{\bern}{\mathop { \text{Bern}}}
\newcommand{\multirank}{collection of rankings}

\title{Nonparametric Preference Completion}
\usepackage{times}

\author{Julian Katz-Samuels}
\author{Clayton Scott \\ \texttt{\{jkatzsam,clayscot\}@umich.edu}}
\affil{Department of Electrical Engineering and Computer Science, \\
University of Michigan}

\begin{document}
\maketitle

\begin{abstract}
We consider the task of collaborative preference completion: given a pool of items, a pool of users and a partially observed item-user rating matrix, the goal is to recover the \emph{personalized ranking} of each user over all of the items. Our approach is nonparametric: we assume that each item $i$ and each user $u$ have unobserved features $x_i$ and $y_u$, and that the associated rating is given by $g_u(f(x_i,y_u))$ where $f$ is Lipschitz and $g_u$ is a monotonic transformation that depends on the user. We propose a $k$-nearest neighbors-like algorithm and prove that it is consistent. To the best of our knowledge, this is the first consistency result for the collaborative preference completion problem in a nonparametric setting. Finally, we demonstrate the performance of our algorithm with experiments on the Netflix and Movielens datasets. 
\end{abstract}

\section{Introduction}

In the preference completion problem, there is a pool of items and a pool of users. Each user rates a subset of the items and the goal is to recover the personalized ranking of each user over all of the items. This  problem is fundamental to recommender systems, arising in tasks such as movie recommendation and news personalization. A common approach is to first estimate the ratings through either a matrix factorization method or a neighborhood-based method and to output personalized rankings from the estimated ratings \citep{koren2009,zhou2008,ning2011,breese1998}. Recent research has observed a number of shortcomings of this approach \citep{weimer2007,liu2008}; for example, many ratings-oriented algorithms minimize the RMSE, which does not necessarily produce a good ranking \citep{cremonesi2010}. This observation has sparked a number of proposals of algorithms that aim to directly recover the rankings \citep{weimer2007,liu2008,lu2014,park2015,oh2015,gunasekar2016}. Although these ranking-oriented algorithms have strong empirical performance, there are few theoretical guarantees to date and they all make specific distributional assumptions (discussed in more detail below). In addition, these results have focused on low-rank methods, while ranking-oriented neighborhood-based methods have received little theoretical attention.

In this paper, we consider a statistical framework for nonparametric preference completion. We assume that each item $i$ and each user $u$ have unobserved features  $x_i$ and $y_u$, respectively, and that the associated rating is given by $g_u(f(x_i,y_u))$ where $f$ is Lipschitz and $g_u$ is a monotonic transformation that depends on the user. We make the following contributions. (i) We propose a simple $k$-nearest neighbors-like algorithm, (ii) we provide, to the best of our knowledge, the first consistency result for ranking-oriented algorithms in a nonparametric setting, and (iii) we provide a necessary and sufficient condition for the optimality of a solution (defined below) to the preference completion problem.

\section{Related Work}
\label{related_work_section}

The two main approaches to preference completion are matrix factorization methods (e.g., low-rank approximation) and neighborhood-based methods. Recently, there has been a surge of research with many theoretical advances in low-rank approximation for collaborative filtering, e.g., \citep{recht2011,keshavan2010}. These methods tend to focus on minimizing the RMSE even though applications usually use ranking measures to evaluate performance.  While recent work has developed ranking-oriented algorithms that outperform ratings-oriented algorithms \citep{gunasekar2016,liu2008,rendle2009,pessiot2007,cremonesi2010,weimer2007}, many of these proposals lack basic theoretical guarantees such as consistency. A recent line of work has begun to fill this gap by establishing theoretical results under specific generative models. \citet{lu2014} and \citet{park2015} provided consistency guarantees using a low-rank approach and the Bradley-Terry-Luce model. Similarly, \citet{oh2015} established a consistency guarantee using a low rank approach and the MultiNomial Logit model. By contrast, our approach forgoes such strong parametric assumptions. 

Neighborhood-based algorithms are popular methods, e.g. \citep{das2007}, because they are straightforward to implement, do not require expensive model-training, and generate interpretable recommendations \citep{ning2011}. There is an extensive experimental literature on neighborhood-based collaborative filtering methods. The most common approach is the user-based model; it is based on the intuition that if two users give similar ratings to items in the observed data, then their unobserved ratings are likely to be similar. This approach  employs variants of $k$ nearest-neighbors. Popular similarity measures include the Pearson Correlation coefficient and cosine similarity. There are a large number of schemes for predicting the unobserved ratings using the $k$ nearest neighbors, including taking a weighted average of the ratings of the users and majority vote of the users \citep{ning2011}. 

\sloppy Recently, researchers have sought to develop neighborhood-based collaborative filtering algorithms that aim to learn a personalized ranking for each user instead of each user's ratings \citep{liu2008,wang2014,wang2016}. Eigenrank, proposed by \citet{liu2008}, is structurally similar to our algorithm. It measures the similarity between users with the Kendall rank correlation coefficient, a measure of the similarity of two rankings. Then, it computes a utility function $\psi : [n_1] \times [n_1] \longrightarrow \bbR$ for each user that estimates his pairwise preferences over the items. From the estimated pairwise preferences, it constructs a personalized ranking for each user by either using a greedy algorithm or random walk model. In contrast, our algorithm uses the average number of agreements on pairs of items to measure similarity between users and a majority vote approach to predict pairwise preferences.

\sloppy Neighborhood-based collaborative filtering has not received much theoretical attention. \citet{kleinberg2003,kleinberg2004} model neighborhood-based collaborative filtering as a latent mixture model and prove consistency results in this specific generative setting. Recently, \citet{lee2016}, who inspired the framework in the current paper, studied rating-oriented neighborhood-based collaborative filtering in a more general nonparametric setting. Their approach assumes that  each item $i$ and each user $u$ have unobserved features $x_i$ and $y_u$, respectively, and that the associated rating is given by $f(x_i,y_u)$ where $f$ is Lipschitz, whereas we assume that the associated rating is given by $g_u(f(x_i,y_u))$ where $g_u$ is a user-specific monotonic transformation. As we demonstrate in our experiments, their algorithm is not robust to monotonic transformations of the columns, but this robustness is critical for many applications. For example, consider the following implicit feedback problem \citep{hu2008}. A recommender system for news articles measures how long users read articles as a proxy for item-user ratings. Because reading speeds and attention spans vary dramatically, two users may actually have very similar preferences despite substantial differences in reading times.

Even though our method is robust to user-specific monotonic transformations, we do not require observing many more entries of the item-user matrix than \citet{lee2016} in the regime where there are many more users than items (e.g., the Netflix dataset). If there are $n_1$ items and $n_2$ users, \citet{lee2016} requires that there exists $\frac{1}{2} > \alpha > 0$ such that the probability of observing an entry is greater than $\max(n_1^{-\frac{1}{2} +\alpha}, n_2^{-1 + \alpha})$, whereas we require that this probability is greater than $\max(n_1^{-\frac{1}{2} +\alpha}, n_2^{-\frac{1}{2} + \alpha})$. 

Our work is also related to the problem of Monotonic Matrix Completion (MMC) where a single monotonic Lipschitz function is applied to a low rank matrix and the goal is rating estimation \citep{ganti2015}. In contrast, we allow for distinct monotonic, possibly non-Lipschitz functions for every user and pursue the weaker goal of preference completion. 

To the best of our knowledge, there is no theoretically supported, nonparametric method for preference completion. Our work seeks to address this issue.


\section{Setup}

\textbf{Notation:} Define $[n] = \set{1, \ldots, n}$. Let $\Omega \subset [n_1] \times [n_2]$. If $X \in \bbR^{n_1 \times n_2}$, let $\sP_{\Omega}(X) \in (\bbR \cup \set{?})^{n_1 \times n_2}$ be defined as $[\sP_{\Omega}(X)]_{i,j} = \left\{
     \begin{array}{lr}
       X_{i,j} & \text{if }  (i,j) \in \Omega \\
      ? & \text{if } (i,j) \not \in \Omega
     \end{array}
   \right.$. 
If $f$ is some function and $U$ a finite collection of objects belonging to the domain of $f$, let ${\max}^{(l)}_{u \in U}f(u)$ denote the $l$th largest value of $f$ over $U$. Let $\bern(p)$ denote a realization of a Bernoulli random variable with parameter $p$. For a metric space $\sM$ with metric $d_{\sM}$, let $B_\epsilon(z) = \set{z^\prime \in \sM : d_{\sM}(z,z^\prime) < \epsilon}$. We use bold type to indicate random variables. For example, $\bz$ denotes a random variable and $z$ a realization of $\bz$.

\textbf{Nonparametric Model:} Suppose that there are $n_1$ items and $n_2$ users. Furthermore,
\begin{enumerate}
\item The items are associated with unobserved features $\bx_1, \ldots, \bx_{n_1} \in \sX$, and the users are associated with unobserved features $\by_1, \ldots, \by_{n_2} \in \sY$ where $\sX$ and $\sY$ are compact metric spaces with metrics $d_\sX$ and $d_\sY$, respectively.
\item \sloppy $\bx_1, \ldots, \bx_{n_1}, \by_1, \ldots, \by_{n_2}$ are independent random variables such that $\bx_1,\ldots, \bx_{n_1} \overset{i.i.d.}{\sim} \sP_{\sX}$ and $\by_1, \ldots, \by_{n_2} \overset{i.i.d.}{\sim} \sP_{\sY}$ where $\sP_{\sX}$ and $\sP_{\sY}$ denote Borel probability measures over $\sX$ and $\sY$, respectively. We assume that for all $\epsilon >0$ and $y \in \sY$, $\sP_{\sY}(B_\epsilon(y)) > 0$. 
\item The complete ratings matrix is $H \coloneq [h_u(x_i, y_u)]_{i \in [n_1], u \in [n_2]}$ where $h_u = g_u \circ f$, $f: \sX \times \sY \longrightarrow \bbR$ is a Lipschitz function with respect to the induced metric $d_{\sX \times \sY}((x_1, y_1), (x_2,y_2)) \coloneq \max(d_{\sX}(x_1, x_2), d_{\sY}(y_1, y_2))$ with Lipschitz constant $1$,\footnote{We could develop our framework with an arbitrary Lipschitz constant $L$, but for ease of presentation, we fix $L = 1$.} i.e., $\forall y_1, y_2 \in \sY$ and $\forall x_1, x_2 \in \sX$, $|f(x_1, y_1) -f(x_2, y_2)| \leq \max(d_\sX(x_1, x_2), d_\sY(y_1, y_2))$,
and $g_u$ is a nondecreasing function.  Note that each $h_u$ need not be Lipschitz. 
\item Each entry of the matrix $H$ is observed independently with probability $p$. Let $\bOmega \subset [n_1] \times [n_2]$ be a random variable denoting the indices of the observed ratings.
\end{enumerate}

\sloppy Whereas \citet{lee2016} considers the task of completing a partially observed matrix $F \coloneq [f(x_i,y_u)]_{i \in [n_1], u \in [n_2]}$ when $\set{x_i}_{i \in [n_1]}$ and $\set{y_u}_{u \in [n_2]}$ are unobserved, we aim to recover the ordering of the elements in each column of $H$ when $\set{x_i}_{i \in [n_1]}$ and $\set{y_u}_{u \in [n_2]}$ are unobserved. In our setup, we view $F$ as an ideal preference matrix representing how much users like items and $H$ as how those preferences are expressed based on user-specific traits (see the news recommender system example in Section \ref{related_work_section}).

This framework subsumes various parametric models. For example, consider a matrix factorization model that assumes that there is a matrix $H \in \bbR^{n_1 \times n_2}$ of rank $d \leq \min(n_1,n_2)$ such that user $u$ prefers item $i$ to item $j$ if and only if $H_{i,u} > H_{j,u}$. Then, we can factorize $H$ such that $H_{i,u} = x_i^t y_u$ where $x_i, y_u \in \bbR^d$ for all $i \in [n_1]$ and $u \in [n_2]$. In our setup, we have $f(x_i,y_u) = x_i^t y_u$ and $g_u(z) = z$.

\textbf{Task:} Let $\sS^{n_1} = \set{ \sigma : \sigma: [n_1] \longrightarrow [n_1], \,  \sigma \text{ is a permutation}}$ denote the set of permutations on $n_1$ objects. We call $\sigma \in \sS^{n_1}$ a \emph{ranking}. Let $\sS^{n_1 \times n_2}  = (\sS^{n_1})^{n_2}$. That is, $\sigma \in \sS^{n_1 \times n_2}$ if $\sigma: [n_1] \times [n_2] \longrightarrow [n_1]$ and for fixed $u \in [n_2]$, $\sigma(\cdot, u)$ is a permutation on $[n_1]$. We call $\sigma \in \sS^{n_1 \times n_2}$ a \emph{\multirank}. Let $\epsilon > 0$. Our goal is to learn $\sigma \in \sS^{n_1 \times n_2}$ that minimizes the number of pairwise ranking disagreements per user with some slack, i.e.,
\begin{align*}
{\dis}_\epsilon(\sigma, H) = \sum_{u=1}^{n_2} \sum_{i < j } &  \ind{|f(x_i,y_u) - f(x_j,y_u)| > \epsilon} \bm{1}\{(h_u(x_i, y_u) - h_u(x_j, y_u)) (\sigma(i,u) - \sigma(j,u)) < 0 \} .
\end{align*}

\section{Algorithm}
\label{algorithm_section}

Our algorithm, Multi-Rank (Algorithm \ref{multi_rank_algorithm}), has two stages: first it estimates the pairwise preferences of each user and, second, it constructs a full ranking for each user from its estimated pairwise preferences. In the first stage, Multi-Rank computes $A \in \set{0,1}^{n_2 \times n_1 \times n_1}$ where $A_{u,i,j} = 1$ denotes that user $u$ prefers item $i$ to item $j$ and $A_{u,i,j} = 0$ denotes that user $u$ prefers item $j$ to item $i$. If a user has provided distinct ratings for a pair of items, Multi-Rank fills in the corresponding entries of $A$. Otherwise, Multi-Rank uses a subroutine called Pairwise-Rank that we will describe shortly. Once Multi-Rank has constructed $A$, it applies the Copeland ranking procedure to the pairwise preferences of each user (discussed at the end of the section).

\begin{algorithm}
\caption{Multi-Rank}
\label{multi_rank_algorithm}
\begin{algorithmic}[1]
\STATE \textbf{Input: $\sP_{\Omega}(H), \beta \geq 2, k > 0$}
\FOR{$u \in [n_2]$, $i,j \in [n_1], i < j$}
\IF{$(i,u) \in \Omega$, $(j,u) \in \Omega$ and $H_{i,u} \neq H_{j,u}$}
\STATE Set $A_{u,i,j} = \ind{H_{i,u} > H_{j,u}}$
\STATE Set $A_{u,j,i} = 1 - A_{u,i,j}$
\ELSE
\STATE Set $A_{u,i,j} = \text{Pairwise-Rank}(u,i,j,\beta, k)$
\STATE Set $A_{u,j,i} = 1 -A_{u,i,j}$ 
\ENDIF
\ENDFOR
\FOR{$u \in [n_2]$}
\STATE $\widehat{\sigma}_u = \text{Copeland}(A_{u,:,:})$
\ENDFOR
\RETURN $\widehat{\sigma} \coloneq (\widehat{\sigma}_1, \ldots, \widehat{\sigma}_{n_2})$
\end{algorithmic}
\end{algorithm}

The Pairwise-Rank algorithm predicts whether a user $u$ prefers item $i$ to item $j$ or vice versa. It is similar to $k$-nearest neighbors where we use the forthcoming ranking measure as our distance measure. Let $N(u)$ denote the set of items that user $u$ has rated, i.e., 
\begin{align*}
N(u) = \set{l : (l,u) \in \Omega},
\end{align*}
and $N(u,v) = N(u) \cap N(v)$ denote the set of items that users $u$ and $v$ have both rated. Viewing $N(u,v)$ as an ordered array where $N(u,v)[\ell]$ denotes the $(\ell+1)$th element, let 
\begin{align*}
I(u,v) = \set{(s,t) : \, & s = N(u,v)[\ell], t = N(u,v)[\ell+1]  \text{ for some } \ell \in \set{2k : k \in \bbN \cup \set{0} }}.
\end{align*}
In words, $I(u,v)$ is formed by sorting the indices of $N(u,v)$ and selecting nonoverlapping pairs in the given order. Note that there is no overlap between the indices in the pairs in $I(u,v)$.\footnote{We select nonoverlapping pairs to preserve independence in the estimates for the forthcoming analysis.} Fix $y_u, y_v \in \sY$. If $I(u,v) = \emptyset$, define $R_{u,v} = 0$ and if $I(u,v) \neq \emptyset$, let $R_{u,v} \coloneq$
\begin{align*}
 \frac{1}{ |I(u,v)| } \sum_{(s,t) \in I(u,v)} & \bm{1}\{(h_u(\bx_s,y_u) - h_u(\bx_t, y_u)) (h_v(\bx_s, y_v) - h_v(\bx_t, y_v)) \geq 0 \}
\end{align*}
denote the fraction of times that users $u$ and $v$ agree on the relative ordering of item pairs belonging to $I(u,v)$. In practice, one can simply compute this statistic over all pairs of commonly rated items. Observe that $\rho(y_u, y_v) \coloneq$
\begin{align*}
 \bbE[R_{u,v} | I(u,v) \neq \emptyset,   \by_u = y_u, \by_v = y_v]  =   {\Pr}_{\bx_s, \bx_t \sim P_\sX} ( &  [h_u(\bx_s, y_u) - h_u(\bx_t, y_u)]  \\
 \times &  [ h_v(\bx_s, y_v) - h_v(\bx_t, y_v)] \geq 0) 
\end{align*}
i.e., $\rho(y_u, y_v)$ is the probability that users $u$ and $v$ with features $y_u$ and $y_v$ order two random items in the same way. 


We apply Pairwise-Rank (Algorithm \ref{pairwise_rank}) to a user $u$ and a pair of items $(i,j)$ if the user has not provided distinct ratings for items $i$ and $j$. Pairwise-Rank($u,i,j,\beta, k$) finds users that have rated items $i$ and $j$, and have rated at least $\beta$ items in common with $u$. If there are no such users, Pairwise-Rank flips a coin to predict the relative preference ordering. If there are such users, then it sorts the users in decreasing order of $R_{u,v}$ and takes a majority vote over the first $k$ users about whether item $i$ or item $j$ is preferred. If the vote results in a tie, Pairwise-Rank flips a coin to predict the relative preference ordering.

\begin{algorithm}
\caption{Pairwise-Rank}
\label{pairwise_rank}
\begin{algorithmic}[1]
\STATE \textbf{Input:} $u \in [n_2], i \in [n_1], j \in [n_1], \beta \geq 2, k \in \bbN$
\STATE $W_u^{i,j}(\beta) = \set{ v \in [n_2] : |N(u,v)| \geq \beta, (i,v), (j,v) \in \Omega}$ 
\STATE Sort $W_u^{i,j}(\beta)$ in decreasing order of $R_{u,v}$ and let $V$ be the first $k$ elements.
\IF{ $V = \emptyset$}
\RETURN $\bern(\frac{1}{2})$
\ENDIF
\STATE $\forall v \in V$, set $P_v = \ind{h_v(x_i,y_v) > h_v(x_j,y_v)} - \ind{h_v(x_i,y_v) < h_v(x_j,y_v)}$
\IF{ $\sum_{v \in V} P_v > 0$}
\RETURN $1$
\ELSIF{$\sum_{v \in V} P_v < 0$}
\RETURN $0$
\ELSE
\RETURN $\bern(\frac{1}{2})$
\ENDIF
\end{algorithmic}
\end{algorithm}

Next, Multi-Rank converts the pairwise preference predictions of each user into a full estimated ranking for each user. It applies the Copeland ranking procedure (Algorithm \ref{Copeland})--an algorithm for the feedback arc set problem in tournaments \citep{copeland1951,coppersmith2006} to each user-specific set of pairwise preferences. The Copeland ranking procedure simply orders the items by the number of times an item is preferred to another item. It is possible to use other approximation algorithms for the feedback arc set problem such as Fas-Pivot from \citet{ailon2008}. 


\begin{algorithm}
\caption{Copeland}
\label{Copeland}
\begin{algorithmic}[1]
\STATE \textbf{Input: $A \in \set{0,1}^{n_1 \times n_1}$}
\FOR{$j \in [n_1]$}
\STATE $I_j =  \sum_{i=1, i \neq j}^{n_1} A_{j,i}$
\ENDFOR
\RETURN $\sigma \in \sS^{n_1}$ that orders items in decreasing order of $I_j$
\end{algorithmic}
\end{algorithm}

\section{Analysis of Algorithm}
\label{analysis_section}

The main idea behind our algorithm is to use pairwise agreements about items to infer whether two users are close to each other in the feature space. However, this is not possible in the absence of further distributional assumptions. The Lipschitz condition on $f$ only requires that if users $u$ and $v$ are close to each other, then $\max_z|f(z, y_u) - f(z, y_v)|$ is small. Proposition \ref{counterexample_prop} shows that there exist functions arbitrarily close to each other that disagree about the relative ordering of almost every pair of points.

\begin{prop}
\label{counterexample_prop}
Let $\sX = [0,1]$ and $\sP_\sX$ be the Lebesgue measure over $\sX$. For every $\epsilon > 0$, there exist functions $f, g : \sX \longrightarrow \bbR$ such that $\max_{x \in [0,1]} |f(x) - g(x)| = \norm{f-g}_\infty \leq \epsilon$ and for almost every pair of points $(x,x^\prime) \in [0,1]^2$, $f(x) > f(x^\prime)$ iff $g(x) < g(x^\prime)$. 
\end{prop}

Thus, we make the following mild distributional assumption.
\begin{defn}
Fix $y \in \sY$ and let $f_y(x) \coloneq f(x,y)$. Let $r$ be a positive nondecreasing function. We say $y$ is \emph{$r$-\diverse} if $\forall \epsilon > 0$, $\Pr_{\bx_1, \bx_2 \sim \sP_\sX}(|f_y(\bx_1) - f_y(\bx_2)| \leq 2 \epsilon) < r(\epsilon)$.
\end{defn}
This assumption says that the probability that $f_y(\bx_1)$ and $f_y(\bx_2)$ are within $\epsilon$ of each other decays at some rate given by $r$. In a sense, it means that users perceive some difference between most randomly selected items with different features, although the difference might be masked by the transformation $g_u$.

We also assume that if two users are not close to each other in the latent space, then they must have some disagreements. Definition \ref{discrim_def} requires that the nonparametric model is economical (i.e., not redundant) in the sense that different parts of the feature space correspond to different preferences.
\begin{defn}
\label{discrim_def}
Fix $y \in \sY$. Let $\epsilon, \delta > 0$. We say that $y$ is \emph{$(\epsilon, \delta)$-\discern} if $z \in B_{\epsilon}(y)^c$ implies that $\rho(y,z) < 1 - \delta$. 
\end{defn}

These assumptions are satisfied under many parametric models. Proposition \ref{example_prop} provides two illustrative examples under a matrix factorization model. We briefly note that, as we show in the supplementary material, $f(x,y) = x^t y$ and $f(x,y) = \norm{x - y}_2$ are equivalent models by adding a dimension.
\begin{prop}
\label{example_prop}
Consider $(\bbR^d, \norm{\cdot}_2)$. Let $f(x,y) = \norm{x - y}_2$ and $g_u(\cdot)$ be strictly increasing $\forall u \in [n_2]$.
\begin{enumerate}
\item Let $\sX = \sY =  \set{ x \in \bbR^d : \norm{x}_2 \leq 1}$, $\sP_\sX$ be the uniform distribution and for all $y \in \sY$ define $r_y(\epsilon) = \sup_{z \in [0,2]} \sP_{\sX}(B_z(y) \setminus B_{z - 4 \epsilon}(y))$. Then, for all $y \in \sY$, $y$ is $r_y$-\diverse. Further, define for all $\epsilon > 0$, $\delta_{\epsilon}= \inf_{v \in \sY} 2\sP_{\sX}(B_{\frac{\epsilon}{2}}(v))^2$. Then, for all $y_u \in \sY$ and for all $\epsilon > 0$, $y_u$ is $(\epsilon, \delta_{\epsilon})$-\discern.

\item Let $\sX \subset \bbR^d$ be a finite collection of points, $\sP_\sX$ be uniform over $\sX$, and for all $y \in \sY$ define $r_y(\epsilon) = \frac{|\set{(x,x^\prime) \in \sX \times \sX : |\norm{y - x}  - \norm{ y -x^\prime}| \leq 2 \epsilon} |}{|\sX|^2}$. Then, for all $y \in \sY$, $y$ is $r_y$-\diverse. Next, suppose $\sY$ is a finite collection of points and every pair of distinct $y,y^\prime \in \sY$ disagree about at least $C$ pairs of items. Let $\delta = \frac{C}{|\sX|^2}$. For all $y_u \in \sY$ and for all $\epsilon > 0$, $y_u$ is $(\epsilon, \delta)$-\discern.
\end{enumerate}

\end{prop}

Our analysis uses two functions to express problem-specific constants. First, let $\tau: \bbR_{++} \longrightarrow (0,1]$ be defined as $\tau(\epsilon) = \inf_{y_0 \in \sY}\Pr_{\by \sim \sP_{\sY}}(d_{\sY}(y_0, \by) \leq \epsilon)$. Second, let $\kappa: \bbR_{++} \longrightarrow (0,1]$ be such that $\kappa(\epsilon) = \inf_{y_0 \in \sY} \Pr_{\by \sim \sP_{\sY}}(d_{\sY}(y_0, \by) > \epsilon)$. Our assumption that for all $\delta >0$ and $y \in \sY$, $\sP_{\sY}(B_\delta(y)) > 0$ ensures that $\tau(\cdot) > 0$ and $\kappa(\cdot) < 1$ (see Lemma \ref{kappa_tau_lemma}). If $\sP_\sY$ is uniform over the unit cube in $(\bbR^d,\norm{\cdot}_\infty)$, then $\tau(\epsilon) = \min(1,\epsilon)^d$ and if $\sY$ is a finite collection of points, then $\tau(\epsilon) = \min_{y \in \sY}\sP_{\sY}(y)$ \citep{lee2016}.


Our model captures the intrinsic difficulty of a problem instance as follows. $r(\cdot)$ and $\tau(\cdot)$ together control the probability of sampling nearby users with similar preferences.  $(\epsilon, \delta)$-\discern \, captures how often users $u$ and $v$ must agree in order to infer that $y_u$ and $y_v$ are close in the latent space and, thus,  $\max_z|f(z, y_u) - f(z, y_v)| \leq \epsilon$.


\subsection{Continuous Ratings Setting}
\label{continuous_rating_main_section}

Our analysis deals with the case of continuous ratings and the case of discrete ratings separately. In this section, we prove theorems dealing with the continuous case and in the next section we give analogous results with similar proofs for the discrete case. Theorem \ref{multi_rank_thm} establishes that with probability tending to $1$ as $n_2 \longrightarrow \infty$, Multi-Rank outputs $\widehat{\sigma} \in \sS^{n_1 \times n_2}$ such that ${\dis}_{2\epsilon}(\widehat{\sigma}, H) = 0$.

\begin{thm}
\label{multi_rank_thm}
 Suppose $\forall u \in [n_2], \,  g_u(z)$ is strictly increasing. Let $\epsilon, \delta > 0$, $\eta \in (0, \frac{\epsilon}{2})$. Suppose that almost every $y \in \sY$ is $(\frac{\epsilon}{2}, \delta)$-\discern . Let $r$ be a positive nondecreasing function such that $r(\frac{\epsilon}{2}) \geq \delta$ and $r(\eta) < \frac{\delta}{2}$. Suppose that almost every $y \in \sY$ is $r$-\diverse . Let $0 < \alpha < \frac{1}{2}$. If $p \geq \max(n_1^{-\frac{1}{2}+ \alpha}, n_2^{-\frac{1}{2}+ \alpha})$, $n_1 p^2 \geq 16$, and $n_2$ is sufficiently large, then Multi-Rank with $k =1$ and $\beta = \frac{p^2 n_1}{2}$ outputs $\widehat{\sigma} \in \sS^{n_1 \times n_2}$ such that
\begin{align*}
{\Pr}_{\set{\bx_i}, \set{\by_u}, \bOmega}({\dis}_{2\epsilon}(\widehat{\sigma}, H) > 0) \leq & n_2 {n_1 \choose 2} [ 2 \exp(- \frac{(n_2 - 1) p^2}{12}) + (n_2 - 1)\exp(-\frac{n_1 p^2 }{8}) \\
 + & \exp(- (\frac{(n_2 - 1) p^2}{2})\tau(\eta))  \\
 + & 3(n_2 - 1) p^2 \exp(- \frac{ \delta^2 n_1 p^2}{20}) ].
\end{align*}
\end{thm}

A couple of remarks are in order. First, if $\epsilon$ and $\delta$ are small, then $\eta$ must be correspondingly small. $\eta$ represents how close a user $y_v$ must be to a user $y_u$ in the feature space to guarantee that the ratings of $y_v$ can be used to make inferences about the ranking of user $y_u$. Second, whereas we require that $p \geq n_2^{-\frac{1}{2} + \alpha}$, \citet{lee2016} require that $p \geq n_2^{-1+ \alpha}$. We conjecture that this stronger requirement is fundamental to our algorithm since $v \in W_u^{i,j}(\beta)$ only if $v$ has rated both items $i$ and $j$, which $v$ does with probability $p^2$. However, there may be another algorithm that circumvents this issue. Theorem \ref{multi_rank_thm} implies the following Corollary.

\begin{cor}
\label{pairwise_rank_cor}
Assume the setting of Theorem \ref{multi_rank_thm}. If $n_2 \longrightarrow \infty$, $p \geq \max(n_1^{-\frac{1}{2} + \alpha}, n_2^{-\frac{1}{2} + \alpha})$, and $n_2^{C_1} \geq n_1 \geq C_2 \log(n_2)^{\frac{1}{2 \alpha}}$ for any constant $C_1 > 0$ and some  constant $C_2 > 0$ depending on $\alpha$, then ${\Pr}_{\set{\bx_i}, \set{\by_u}, \Omega}({\dis}_{2\epsilon}(\widehat{\sigma}, H) > 0) \longrightarrow 0$ as $n_2 \longrightarrow \infty$. 
\end{cor}

Note that the growth rates of $n_1, n_2$ and $p$ imply that the average number of rated items by each user $p n_1$ must grow as $C \log(n_2)^{\frac{1}{2} + \frac{1}{4 \alpha}}$ for some universal constant $C > 0$. 

Next, we sketch the proof. The main part of the analysis deals with establishing a probability bound of a mistake by Pairwise-Rank for a specific user $u$ and a pair of items $i$ and $j$ when $|f(\bx_i, \by_u) - f(\bx_j, \by_u)| > \epsilon$. First, we establish that w.h.p. $|W_u^{i,j}(\beta)|$ is large, i.e., there are many users that have rated $i$ and $j$ and many other items in common with $u$. Second, using standard concentration bounds, it is shown that for every $v \in W_u^{i,j}(\beta)$, $R_{u,v}$ concentrates around $\rho(u,v)$. Since $\beta \longrightarrow \infty$, this estimate converges to $\rho(u,v)$. Third, we show that eventually we sample a point from $B_\eta(\by_u)$. Further, if $\by_v \in B_\eta(\by_u)$ and $\by_w \in B_{\frac{\epsilon}{2}}(\by_u)^c$ (note $\eta \leq \frac{\epsilon}{2}$), then since $\by_u$ is $(\frac{\epsilon}{2}, \delta)$-\discern \, w.p. 1, by our choice of $\eta$, $\rho(\by_u, \by_v) > \rho(\by_u, \by_w) + \frac{\delta}{2}$. Thus, by concentration bounds, $R_{u,v} > R_{u,w}$. Therefore, Pairwise-Rank with $k=1$ uses the preference ordering of a user in $B_{\frac{\epsilon}{2}}(\by_u)$ on items $i$ and $j$ to make the prediction. The Lipschitzness of $f$ and our assumption that $g_v$ is strictly increasing imply that this prediction is correct. It is possible to extend this argument to handle the case when $k > 1$. 

\subsection{Discrete Ratings Setting}
\label{discrete_section}

Let $N > 0$ and suppose that $|f(x,y)| \leq N$  $\forall x \in \sX$,  $\forall y \in \sY$. Suppose that there are $L$ distinct ratings and let $\sG$ denote the set of all step functions of the form
\begin{align*}
g_u(x) = \left\{
     \begin{array}{lr}
       1 & : x \in [-N, a_{u,1})\\
      2  & : x \in [a_{u,1}, a_{u,2}) \\
      \vdots & \\
      L  & : x \in [a_{u,L-1}, N] \\
     \end{array}
   \right. .
\end{align*}
We assume that for all $u \in [n_2]$, $g_u \in \sG$ and that the rating thresholds are random, i.e., $(\ba_{1,1}, \ldots, \ba_{1,L-1}), \ldots,(\ba_{n_2,1}, \ldots, \ba_{n_2,L-1})  \overset{i.i.d.}{\sim} \sP_{[-N, N]^{L-1}}$. We write $\bg_1, \ldots, \bg_{n_2} \overset{i.i.d.}{\sim} \sP_{\sG}$ and we assume that $\set{\bg_u}_{u \in [n_2]}$ is independent from $\set{\bx_i}_{i \in [n_1]}$, $\set{\by_u}_{u \in [n_2]}$, and $\bOmega$. Let $\sP_l$ denote the marginal distribution of $\ba_{u,l}$ for all $u \in [n_2]$. We make the following assumption.
\begin{defn}
We say that $\sP_{\sG}$ is \emph{diverse} if for every open interval $I \subset [-N, N]$ there exists $l$ such that $\sP_l(I) > 0$.
\end{defn}
Let $d_{\bbR}$ denote a metric on $\bbR$; fix $u \in [n_2]$ and let $\gamma(\epsilon) = \inf_{z \in [-N,N]} \sP_{\set{\ba_{u,l}}_{l \in [L-1]}}(\exists l \in [L-1] :  d_\bbR(z,\ba_{u,l}) \leq \epsilon)$. The aforementioned assumption ensures via a measure theoretic argument that $\gamma(\epsilon) > 0$ for all $\epsilon > 0$ (see Lemma \ref{kappa_tau_lemma} in the Appendix).

\begin{thm}
\label{multi_rank_discrete_thm}
Let $\epsilon, \delta > 0$ and $\eta \in (0, \frac{\epsilon}{4})$. Suppose that $\sP_{\sG}$ is diverse and that almost every $y \in \sY$ is $(\frac{\epsilon}{4}, \delta)$-\discern . Let $r$ be a positive nondecreasing function such that $r(\frac{\epsilon}{4}) \geq \delta$ and $r(\eta) < \frac{\delta}{2}$. Suppose that almost every $y \in \sY$ is $r$-\diverse . Let $\frac{1}{2} > \alpha > \alpha^\prime > 0$. If $p \geq \max(n_1^{-\frac{1}{2}+ \alpha}, n_2^{-\frac{1}{2}+ \alpha})$, $n_1 p^2 \geq 16$, $n_1 \geq C_1 \log(n_2)^{\frac{1}{2 \alpha}}$ for some constant $C_1$, and $n_2$ is sufficiently large, Multi-Rank with $k = n_2^{\alpha^\prime}$ and $\beta = \frac{p^2 n_1}{2}$ outputs $\widehat{\sigma}$ such that
\begin{align*}
{\Pr}_{\set{\bx_i}, \set{\by_u}, \set{\ba_{u,l}}, \bOmega} ( {\dis}_{2\epsilon}(\widehat{\sigma}, H) > 0) 
\leq & n_2 {n_1 \choose 2} [2 \exp(- \frac{(n_2 - 1) p^2}{12}) + (n_2 - 1)\exp(-\frac{n_1 p^2 }{8}) \\
+ & 2  \exp(-\gamma(\frac{\epsilon}{4}) k ) \\
+ &  \frac{1}{ 1 - r(\frac{\epsilon}{2})} [3(n_2 - 1) p^2 \exp(- \frac{ \delta^2 n_1 p^2}{20}  )    \\
+ & \exp([1  - \kappa(\frac{\epsilon}{4}) + \tau(\eta)+  \log(3 \frac{(n_2 -1)p^2}{2})]k \\
- & k\log(k)  - \tau(\eta)\frac{(n_2 -1)p^2}{2}) ]].
\end{align*}
\end{thm}
\begin{cor}
\label{pairwise_rank_cor_discrete}
Assume the setting of Theorem \ref{multi_rank_discrete_thm}.
 If $p \geq \max(n_1^{-\frac{1}{2}+ \alpha}, n_2^{-\frac{1}{2}+ \alpha})$, $k = n_2^{\alpha^\prime}$, and $n_2^{C_1} \geq n_1 \geq C_2 \log(n_2)^{\frac{1}{2 \alpha}}$ for any constant $C_1 > 0$ and some  constant $C_2 > 0$ depending on $\alpha$, then ${\Pr}_{\set{\bx_i}, \set{\by_u}, \Omega}({\dis}_{2\epsilon}(\widehat{\sigma}, H) > 0) \longrightarrow 0$ as $n_2 \longrightarrow \infty$. 
\end{cor}

The bulk of the analysis for the discrete ratings setting is similar to the continuous rating setting and, once again, mainly deals with the analysis of Pairwise-Rank for a user $u$ and items $i$ and $j$. Since the ratings are discrete, although users that are sufficiently close to user $u$ in the feature space agree about the ordering of items $i$ and $j$, we need to show that at least one of these neighbors does not give the same rating to items $i$ and $j$. To this end, we show that eventually $k$ nearby points are sampled: $\by_{v_1}, \ldots, \by_{v_k} \in B_\eta(\by_u)$. Conditional on $|f(\bx_i,\by_u) - f(\bx_j,\by_u)| > \epsilon$, using the Lipschitzness of $f$, $(f(\bx_i, \by_{v_q}), f(\bx_j,\by_{v_q}))$ has length at least $\frac{\epsilon}{2}$. Finally, since $\sP_{\sG}$ is diverse, a concentration argument wrt $\bg_{v_1}, \ldots, \bg_{v_k}$ implies that w.h.p. there exists $q \in [k]$ and $l \in [L-1]$ such that $\ba_{v_q,l} \in (f(\bx_i, \by_{v_q}), f(\bx_j,\by_{v_q}))$. Thus, user $v_q$ provides distinct ratings for items $i$ and $j$. 

\section{A Necessary and Sufficient Condition for ${\dis}_\epsilon(\sigma, H) = 0$}
\label{necessary_sufficient_section}

In this section, we characterize the class of optimal collections of rankings, i.e., $\sigma \in \sS^{n_1 \times n_2}$ such that ${\dis}_\epsilon(\sigma, H) = 0$. We show roughly that a collection of rankings $\sigma$ is optimal in the sense that ${\dis}_\epsilon(\sigma, H) = 0$ if and only if $\sigma$ agrees with the observed data and $\sigma$ gives the same ranking to users that are close to each other in the latent space $\sY$.  To study this question, we consider the regime where the number of items $n_1$ is fixed, the probability of an entry being revealed $p$ is fixed, and the number of users $n_2$ goes to infinity. 

Consider the following notion, which is the main ingredient in our necessary and sufficient condition:


\begin{defn}
Let $\epsilon > 0$ and $T \subset [n_1] \times [n_1] \times [n_2]$. $\sigma \in \sS^{n_1 \times n_2}$ is an $\epsilon$-\local \, \multirank \, over $T$ if $ \forall i \neq j \in [n_1], u \neq v \in [n_2]$ such that $(i,j,u), (i,j,v) \in T$ and $d_{\sY}(y_u, y_v) \leq \epsilon$, it holds that $\sigma(i,u) < \sigma(j,u) \Longleftrightarrow \sigma(i,v) < \sigma(j,v)$. If $\sigma$ is an $\epsilon$-\local \, \multirank \, over $[n_1] \times [n_1] \times [n_2]$, then we simply say that $\sigma$ is an $\epsilon$-\local \, \multirank.
\end{defn}
In words, a \multirank \, is $\epsilon$-\local \, if it gives the same ranking to users that are within $\epsilon$ of each other in the latent space.

We introduce the following objective function:
\begin{align*}
\edis(\sigma, H) \coloneq \sum_{u=1}^{n_2} \sum_{i < j :(i,u), (j,u) \in \Omega}  \bm{1}\{& (h_u(x_i, y_u) - h_u(x_j, y_u) ) (\sigma(i,u) - \sigma(j,u)) < 0\}. 
\end{align*}
Once again, we analyze separately the continuous rating and discrete rating settings. With respect to the continuous rating setting, Theorems \ref{sufficient_condition_thm_continuous} and \ref{necessary_condition_thm} \emph{roughly} imply that with probability tending to $1$ as $n_2 \longrightarrow \infty$, a \multirank \, $\sigma \in \sS^{n_1 \times n_2}$ that minimizes $\edis(\cdot, H)$ is $\frac{\epsilon}{2}$-\local \, if and only if ${\dis}_\epsilon(\sigma, H) = 0$. A similar statement holds for the discrete rating setting.

To begin, we present our sufficient conditions.

\begin{thm}
\label{sufficient_condition_thm_continuous}
Assume the continuous rating setting. Let $\epsilon > 0$ and suppose that for all $u \in [n_2]$, $g_u(\cdot)$ is strictly increasing. With probability increasing to $1$ as $n_2 \longrightarrow \infty$, if $\sigma \in \sS^{n_1 \times n_2}$ is $\frac{\epsilon}{2}$-\local \, and minimizes $\edis(\cdot, H)$, then ${\dis}_\epsilon(\sigma, H) = 0$.
\end{thm}

\begin{thm}
\label{sufficient_condition_thm_discrete}
Assume the discrete rating setting and that $\sP_{\sG}$ is diverse. Let $\epsilon > 0$. With probability increasing to $1$ as $n_2 \longrightarrow \infty$, if $\sigma \in \sS^{n_1 \times n_2}$ is $\frac{\epsilon}{8}$-\local \, and minimizes $\edis(\cdot, H)$, then ${\dis}_\epsilon(\sigma, H) = 0$.
\end{thm}

The proofs for the continuous and discrete cases are similar. We briefly sketch the argument for the continuous case. Since $\sY$ is compact, there is a finite subcover of $\sY$ with open balls with diameter at most $\frac{\epsilon}{2}$. As $n_2 \longrightarrow \infty	$, with probability increasing to $1$, for every open ball $\sO$ belonging to the finite subcover and for every pair of distinct items $i,j \in [n_1]$, there is some user $u \in \sO$ that has rated $i$ and $j$. Then, on this event, it can be shown that if ${\dis}_\epsilon(\sigma, H) > 0$, then $\edis(\sigma, H) > 0$. Thus, using the contrapositive, the result follows. 

Theorem \ref{necessary_condition_thm} gives our necessary condition. 

\begin{thm}
\label{necessary_condition_thm}
Let $\epsilon > 0$ and $\sigma \in \sS^{n_1 \times n_2}$ such that ${\dis}_\epsilon(\sigma, H) = 0$. Let $T = \set{(i,j,u)\in [n_1] \times [n_1] \times [n_2] : |f(x_i, y_u) - f(x_j, y_u)| > \epsilon, \, h(x_i,y_u) \neq h(x_j,y_u)}$. Then, $\sigma$ is an $\epsilon$-\local \, \multirank \, over $T$. 
\end{thm}
Theorem \ref{necessary_condition_thm}  shows that in our general setting, learning the correct collection of rankings requires giving the same ranking to nearby users. In particular, this provides an intuition on the kind of embedding that matrix factorization learns. Theorem \ref{necessary_condition_thm} only applies to items $i,j$ and user $u$ if there is a large enough difference in the underlying values given by $f$. The proof follows by the Lipschitzness of $f$ and algebra.

\section{Experiments}

\begin{table*}[ht]
	\centering
\scalebox{0.75}{
	\input{results_avg_standard.txt}
}
	\caption{Netflix and MovieLens Results. On the Netflix dataset, MR usually used $\beta = 5$ and $k \in [13,19]$. MRW usually used $\beta = 9$ and $k \in [16,23]$. On the MovieLens dataset, MR usually used $\beta = 10$ and $k \in [7,13]$. MRW usually used $\beta = 12$ and $k \in [13,17]$.}
	\label{netflix_movielens_table}
\end{table*}

\begin{table*}[ht]
	\centering
\scalebox{0.75}{
	\input{results_avg_transformed.txt}
}
	\caption{Quantized Netflix and MovieLens Results. On the Netflix dataset, MR usually used $\beta = 5$ and $k = 22$. MRW usually used $\beta  \in [9,10]$ and $k \in [27,31]$. On the MovieLens dataset, MR usually used $\beta \in [10,13]$ and $k \in [10,19]$. MRW usually used $\beta \in [8,11]$ and $k \in [16,23]$.}
	\label{quantized_netflix_movielens_table}
\end{table*}

\begin{table*}[!htbp]
	\centering
\scalebox{0.75}{
	\input{results_avg_monotonic.txt}
}
	\caption{Monotonically Transformed Netflix and MovieLens Results. We only display the results for LA since the other methods are invariant to monotonic transformations of the columns.}
	\label{monotonic_netflix_movielens_table}
\end{table*}

In this section, we examine the empirical performance of Multi-Rank. It is well-known that matrix factorization methods tend to outperform neighborhood-based methods. Nevertheless, neighborhood-based methods remain popular in situations where practitioners want an easy-to-implement method, to avoid expensive model-building, and to be able to interpret predictions easily \citep{ning2011}. Furthermore, it has been observed that for the task of matrix completion, (i) matrix factorization methods and neighborhood-based methods have complementary strengths and weaknesses and (ii) performance gains can be achieved by merging these methods into a single algorithm \citep{bell2007, koren2008}. Yet, it is non-trivial to generalize ideas for combining matrix factorization and neighborhood-based methods in the matrix completion setting to the preference completion setting. In light of this discussion, the purpose of our experiments is not to demonstrate the superiority of our method over matrix factorization methods, but to compare the performance of our algorithm with the state-of-the-art.

We compared the performance of our algorithm (MR) and a weighted version of our algorithm (MRW) where votes are weighted by $R_{u,v}$ against Alternating SVM (AltSVM) \citep{park2015}, Retargeted Matrix Completion (RMC) \citep{gunasekar2016}, and the proposed algorithm in \citep{lee2016} (LA). We chose AltSVM and RMC because they are state-of-the-art matrix factorization methods for preference completion and we chose LA because its theoretical guarantees are similar to our guarantees for Multi-Rank and it was shown to be superior to item-based and user-based neighborhood methods \citep{lee2016}. We used grid search to optimize the hyperparameters for each of the algorithms using a validation set. 

We use the ranking metrics Kendall Tau, Spearman Rho, NDCG@5, and Precision@5. Kendall Tau and Spearman Rho measure how correlated the predicted ranking is with the true ranking. The other metrics measure the quality of the predicted ranking at the top of the list. For Precision@5, we deem an item relevant if it has a score of $5$. For all of these metrics, higher scores are better. See \citet{liu2009} for a more detailed discussion of these metrics. The numbers in parentheses are standard deviations.

We use the Netflix and MovieLens 1M datasets. We pre-process the data in a similar way to \citet{liu2008}. For the Netflix dataset, we take the $2000$ most popular movies and randomly selected $4000$ users that had rated at least $100$ of these movies. For both datasets, we randomly subsample the ratings $5$ times in the following way: we randomly shuffled the (user-id, movie, rating) triples and split 40\% into a training set, 15\% into a validation set, and 45\% into a test set. For the Netflix dataset, we drop users if they have fewer than 50 ratings in the training set and fewer than 10 ratings in either the validation set or the test set. For the MovieLens dataset, we drop users if they have fewer than 100 ratings in the training set and fewer than 50 ratings in either the validation set or the test set. Table \ref{netflix_movielens_table} shows that although MRW does not have the best performance, it outperforms AltSVM on NDCG@5 and Precision@5 on the Netflix dataset and LA on NDCG@5 and Precision@5 on the MovieLens dataset. 

In addition, we quantized the scores of both datasets to $1$ if the true rating is less than or equal to $3$ and to $5$ otherwise (see Table \ref{quantized_netflix_movielens_table}). Here, MR and MRW have the same amount of information as LA and RMC. On the Netflix dataset, MRW performed the best on the NDCG@5 measure. 

Finally, we considered a setting where a company performs $A/B$ testing on various rating scales (e.g., 1-5, 1-10, 1-50, 1-100) and wishes to use all of the collected data to predict preferences. To model this situation, for each user, we randomly sampled a number $a \in \set{1,2,10,20}$ and $b \in [a-1] \cup \set{0}$, and transformed the rating $r \mapsto a \cdot r - b$. Table \ref{monotonic_netflix_movielens_table} shows that on the monotonically transformed versions of the Netflix and MovieLens datasets, LA performs dramatically worse. This is unsurprising since it is well-known that the performance of rating-oriented neighborhood-based methods like LA suffers when there is rating scale variance \citep{ning2011}.

%

\pagebreak

\appendix

\section{Outline}

In Section \ref{sup_analysis_section}, we give the counterexample establishing Proposition \ref{counterexample_prop} and give theorem proofs for the continuous rating setting. In Section \ref{sup_discrete_section}, we give theorem proofs for the discrete rating setting. In Section \ref{technical_lemmas_section}, we prove the lemmas used in our theorem proofs, beginning with lemmas common to both the continuous rating setting and discrete rating setting and, then, presenting the lemmas on the continuous rating setting and discrete rating setting, separately. In Section \ref{necessary_sufficient_section_app}, we provide the proofs of the necessary and sufficient conditions. 
In Section \ref{example_sec}, we prove Proposition \ref{example_prop} and that the models $f(x,y) = x^t y$ and $f(x,y) = \norm{x-y}_2$ are equivalent by adding a dimension. Finally, in Section \ref{useful_bounds_section}, we give some bounds that we use in the proofs for reference.

Unless otherwise indicated, all probability statements are with respect to $\set{\bx_i}_{i \in [n_1]} \cup \set{\by_u}_{u \in [n_2]} \cup \bOmega$ in the continuous ratings setting and with respect to $\set{\bx_i}_{i \in [n_1]} \cup \set{\by_u}_{u \in [n_2]} \cup \set{\ba_{u,l}}_{u \in [n_2], l \in [L-1]} \cup \bOmega$ in the discrete ratings setting.

\section{Proofs for Section \ref{continuous_rating_main_section} }
\label{sup_analysis_section}

To begin, we introduce some additional notation. When $\by_u$ and $\by_v$ are random, we  write $R_{\bu, \bv}$ instead of $R_{u,v}$ for emphasis. 

\begin{proof}[Proof of Proposition \ref{counterexample_prop}]
Consider the functions
\begin{align*}
   f(z) = \left\{
     \begin{array}{lr}
        \epsilon z & : z \in [0, \frac{1}{2}]\\
       \epsilon (1 - z) & : z \in (\frac{1}{2}, 1]
     \end{array}
   \right.
\end{align*}
and 
\begin{align*}
   g(z) = \left\{
     \begin{array}{lr}
       -  \epsilon z & : z \in [0, \frac{1}{2}]\\
       \epsilon (z - 1) & : z \in (\frac{1}{2}, 1]
     \end{array}
   \right.
\end{align*}
\end{proof}

Next, we analyze Pairwise-Rank (PR), bounding the probability that Pairwise-Rank cannot distinguish between items $i$ and $j$ when $|f(x_i, y_u) - f(x_j,y_u)| > \epsilon$, i.e.,  the event
\begin{align*}
D^\epsilon_{u,i,j} & \coloneq \set{f(\bx_i, \by_u) + \epsilon < f(\bx_j, \by_u)} \cap \set{\text{PR}(u,i,j,\beta, k) = 1)}) \\
& \cup  \set{f(\bx_i, \by_u)  > f(\bx_j, \by_u)+ \epsilon} \cap \set{\text{PR}(u,i,j,\beta, k) = 0}).
\end{align*}

\begin{thm}
\label{pairwise_rank_thm}
 Suppose $\forall u \in [n_2], \,  g_u(z)$ is strictly increasing. Let $\epsilon, \delta > 0$ and $\eta \in (0, \frac{\epsilon}{2})$. Suppose that almost every $y \in \sY$ is $(\frac{\epsilon}{2}, \delta)$-\discern . Let $r$ be a positive nondecreasing function such that $r(\frac{\epsilon}{2}) \geq \delta$ and $r(\eta) < \frac{\delta}{2}$. Suppose that almost every $y \in \sY$ is $r$-\diverse . Let $0 < \alpha < \frac{1}{2}$. If $p \geq \max(n_1^{-\frac{1}{2}+ \alpha}, n_2^{-\frac{1}{2}+ \alpha})$, $n_1 p^2 \geq 16$, and $n_2$ is sufficiently large, for all $u \in [n_2]$ and $i \neq j \in [n_1]$, the output of Pairwise-Rank with $k =1$ and $\beta = \frac{p^2 n_1}{2}$ is such that
\begin{align*}
{\Pr}_{\set{\bx_i}, \set{\by_u}, \bOmega} (D^\epsilon_{u,i,j}) & \leq 2 \exp(- \frac{(n_2 - 1) p^2}{12}) + (n_2 - 1)\exp(-\frac{n_1 p^2 }{8}) \\
+&  \exp(- (\frac{(n_2 - 1) p^2}{2})\tau(\eta))  + 3(n_2 - 1) p^2 \exp(- \frac{ \delta^2 n_1 p^2}{20}).
\end{align*}
\end{thm}

The structure of the proof of Theorem \ref{pairwise_rank_thm} is similar to the proof of Theorem 1 from \citet{lee2016}. The lemmas are distinct, however.

\begin{proof}[ Proof of Theorem \ref{pairwise_rank_thm}]
Fix $u \in [n_2]$, $i,j \in [n_1]$ such that $i \neq j$. Define:
\begin{align*}
W^{i,j}_u(\beta) & = \set{ v \in [n_2] : |N(u,v)| \geq \beta, (i,v), (j,v) \in \Omega}. 
\end{align*}
Further, define the events:
\begin{align*}
A & = \set{|W^{i,j}_u(\beta)| \in [\frac{(n_2 - 1) p^2}{2}, \frac{3(n_2 - 1) p^2}{2}]}, \\
B & = \set{\max_{v \in W^{i,j}_u(\beta)} \rho(\by_u, \by_v) \geq 1 - \frac{\delta}{2}}, \\
C & = \set{ |R_{\bu \bv} - \rho(\by_u, \by_v) | \leq \frac{\delta}{4}, \, \forall v \in W^{i,j}_u(\beta)}.
\end{align*}
By several applications of the law of total probability, we have that
\begin{align*}
\Pr( D_{u,i,j}^\epsilon) & = \Pr( D_{u,i,j}^\epsilon | A, B, C) \Pr(A, B,C) + \Pr( D_{u,i,j}^\epsilon | (A \cap B \cap C)^c) \Pr( (A \cap B \cap C)^c ) \\
& \leq \Pr( D_{u,i,j}^\epsilon | A, B, C) +   \Pr(A^c) + \Pr((A \cap B \cap C)^c | A) \\
& \leq \Pr(D_{u,i,j}^\epsilon | A, B, C) + \Pr( A^c) + \Pr( B^c | A) + \Pr( C^c | A, B).
\end{align*}
We will upper bound each term in the above bound. By Lemma \ref{error_bound_continuous}, $\Pr(D_{u,i,j}^\epsilon | A, B, C)  = 0$.

Setting $\lambda = \frac{1}{2}$ in Lemma \ref{sufficient_overlap} yields that
\begin{align*}
\Pr(A^c) & = \Pr(|W_u^{i,j}(\beta)| \not \in [\frac{(n_2 - 1) p^2}{2}, \frac{3(n_2 - 1) p^2}{2}]) \\
& \leq 2 \exp(- \frac{(n_2 - 1) p^2}{12}) + (n_2 - 1)\exp(-\frac{n_1 p^2 }{8}).
\end{align*}
Lemma \ref{good_neighbor} yields that
\begin{align}
\Pr(B^c | A) & =  \Pr( \max_{v \in W^{i,j}_u(\beta)} \rho(\by_u, \by_v) < 1 - \frac{\delta}{2} |A ) \leq \Pr( \max_{v \in  W^{i,j}_u(\beta)} \rho(\by_u,\by_v) < 1 - r(\eta) |A) \nonumber \\
& \leq [1 - \tau(\eta)]^{\frac{(n_2 - 1) p^2}{2}} \label{B_complement_1_pairwise_rank} \\
&\leq \exp(- (\frac{(n_2 - 1) p^2}{2})\tau(\eta)). \label{B_complement_2_pairwise_rank}
\end{align}
Line \eqref{B_complement_1_pairwise_rank} follows by Lemma \ref{good_neighbor} since conditional on $A$, $W_{u}^{i,j}(\beta) \geq \frac{(n_1-1)p^2}{2}$ and line \eqref{B_complement_2_pairwise_rank} follows by the inequality $1-x \leq \exp(-x)$. Since by hypothesis $\alpha \in (0,\frac{1}{2})$ is fixed such that $p \geq \max(n_1^{-\frac{1}{2}+ \alpha}, n_2^{-\frac{1}{2}+ \alpha})$, there exists a sufficiently large $n_2$ such that line \eqref{B_complement_2_pairwise_rank} is less than $\frac{1}{2}$. Then, by Bayes rule, the union bound, and Lemma \ref{R_concentration},
\begin{align*}
\Pr(C^c | A, B) & \leq \frac{ \Pr(C^c |A)}{\Pr(B |A)} \leq 2 \Pr(C^c |A) \\
& = 2 \Pr( \exists v \in W_u^{i,j}(\beta), |R_{\bu \bv} - \rho(\by_u , \by_v) | > \frac{\delta}{4} | A) \\
& \leq 3(n_2 - 1) p^2 \exp(- \frac{\delta^2}{4} \left \lfloor\frac{\beta}{2} \right \rfloor ) \\
& = 3(n_2 - 1) p^2 \exp(- \frac{\delta^2}{4} \left \lfloor\frac{n_1 p^2}{4} \right \rfloor ) \\
& \leq 3(n_2 - 1) p^2 \exp(- \frac{ \delta^2 n_1 p^2}{20}  ) 
\end{align*}
where the last line follows because $n_1 p^2 \geq 16$ and $\forall x \geq 16$, $\left \lfloor \frac{x}{4} \right \rfloor \geq \frac{x}{5}$. Putting it all together, we have
\begin{align*}
\Pr(D_{u,i,j}^\epsilon) & \leq  2 \exp(- \frac{(n_2 - 1) p^2}{12}) + (n_2 - 1)\exp(-\frac{n_1 p^2 }{8}) \\
&  + \exp(- (\frac{(n_2 - 1) p^2}{2})\tau(\eta))   + 3(n_2 - 1) p^2 \exp(- \frac{ \delta^2 n_1 p^2}{20}) 
\end{align*}
\end{proof}

\begin{proof}[Proof of Theorem \ref{multi_rank_thm}]
For any $u \in [n_2]$, $i \neq j \in [n_1]$, define the event
\begin{align*}
\text{Error}^\epsilon_{u,i,j} & = ( \set{f(\bx_i, \by_u) + \epsilon < f(\bx_j, \by_u)} \cap \set{A_{u,i,j} = 1}) \\
& \cup  ( \set{f(\bx_i, \by_u)  > f(\bx_j, \by_u)+ \epsilon} \cap \set{A_{u,i,j} = 0}).
\end{align*}
Suppose that there exists $u \in [n_2]$ and distinct $i, j \in [n_1]$ such that $ \text{Error}^\epsilon_{u,i,j}$ occurs. Without loss of generality suppose that $f(\bx_i, \by_u) + \epsilon < f(\bx_j, \by_u)$, and $A_{u,i,j} = 1$. Then, inspection of the Multi-Rank algorithm reveals that $1 = A_{u,i,j} = \text{Pairwise-Rank}(u,i,j,\beta, k)$. Thus, $D_{u,i,j}^\epsilon$ occurs. Therefore, by Theorem \ref{pairwise_rank_levels_thm} and the union bound, 
\begin{align*}
& \Pr(\exists u \in [n_2], i \neq j \in [n_1] \text{ s.t. } \text{Error}^\epsilon_{u,i,j}) \\
&   \qquad \leq \Pr(\exists u \in [n_2], i \neq j \in [n_1] \text{ s.t. } D^\epsilon_{u,i,j}) \\
&   \qquad \leq n_2 {n_1 \choose 2} [ 2 \exp(- \frac{(n_2 - 1) p^2}{12}) + (n_2 - 1)\exp(-\frac{n_1 p^2 }{8}) \\
&   \qquad + \exp(- (\frac{(n_2 - 1) p^2}{2})\tau(\eta))   + 3(n_2 - 1) p^2 \exp(- \frac{ \delta^2 n_1 p^2}{20})].
\end{align*}
Now, suppose that $\forall u \in [n_2]$ and $i,j \in [n_1]$ such that $i \neq j$, $(\text{Error}^\epsilon_{u,i,j})^c$ occurs. Then, by Lemma \ref{copeland_lemma}, $\widehat{\sigma} = (\widehat{\sigma}_1, \ldots, \widehat{\sigma}_{n_2})$ with $\widehat{\sigma}_u = \text{Copeland}(A_{u,:,:})$ satisfies ${\dis}_{2\epsilon}(\widehat{\sigma}, H) = 0$.
\end{proof}

\begin{proof}[Proof of Corollary \ref{pairwise_rank_cor}]
Ignoring constants, the two dominant terms in the bound in Theorem \ref{multi_rank_thm} are of the form $n_1^2 n_2 \exp( -n_2 p^2)$ and $n_1^2 n^2_2 \exp(-n_1 p^2)$. Then, under the conditions of Theorem \ref{pairwise_rank_thm}, as $n_2 \longleftarrow \infty	$
\begin{align*}
n_1^2 n_2  \exp( -n_2 p^2) &  \leq \exp(2 \log(n_1) + \log(n_2)-n_2^{2\alpha} ) \\
& \leq \exp((1 + 2 C_1) \log(n_2) - n_2^{2\alpha} ) \longrightarrow 0.
\end{align*}
Now, observe that
\begin{align*}
n_1^2  n_2^2 \exp( -n_1 p^2) & = \exp(2\log(n_2) + 2\log(n_1) - n_1 p^2) \\
& \leq \exp(2\log(n_2) + 2\log(n_1) - n_1^{2 \alpha}) \\
& \leq \exp(4  \max(\log(n_2), \log(n_1)) -  n_1^{2 \alpha}) \\
\end{align*}
Suppose that $n_1 \geq n_2$. Then, clearly, the limit of the RHS as $n_2 \longrightarrow \infty$ is $0$. Now, suppose that $n_1 < n_2$. Then, if $C_2^{2 \alpha} > 4$, then as $n_2 \longrightarrow \infty$,
\begin{align*}
n_2^2 n_1^2 \exp( -n_1 p^2) & \leq \exp(4 \log(n_2) - n_1^{2 \alpha}) \\
& \leq \exp([4 -  C_2^{2 \alpha}] \log(n_2)) \longrightarrow 0. \\
\end{align*}
\end{proof}

\section{Proofs for Section \ref{discrete_section}}
\label{sup_discrete_section}

To begin, because the model for the discrete ratings section is different, we introduce new notation in the interest of clarity.  Fix $y_u, y_v \in \sY$. Define
\begin{align*}
\rho^\prime(y_u, y_v) = {\Pr}_{\bg_u, \bg_v, \bx_s, \bx_t}[\bg_u(f(\bx_s, y_u)) - \bg_u(f(\bx_t, y_u))][\bg_v(f(\bx_s, y_v)) - \bg_v(f(\bx_t, y_v))] \geq 0).
\end{align*}

Note that in this setting, the meaning of $(\epsilon, \delta)$-\discern \, is slightly different.
\begin{defn}
Fix $y \in \sY$. Let $\epsilon, \delta > 0$. We say that $y$ is \emph{$(\epsilon, \delta)$-\discern} if $z \in B_{\epsilon}(y)^c$ implies that $\rho^\prime(y,z) < 1 - \delta$. 
\end{defn}
In a sense, the notion requires in addition that the distribution of the monotonic functions reveals some differences in the preferences of the users.

Unless otherwise indicated, all probability statements are with respect to $\set{\bx_i}_{i \in [n_1]} \cup \set{\by_u}_{u \in [n_2]} \cup \set{\ba_{u,l}}_{u \in [n_2], l \in [L-1]} \cup \bOmega$. Next, we prove a theorem that is analogous to Theorem \ref{pairwise_rank_thm}. Recall the notation:
\begin{align*}
D^\epsilon_{u,i,j} & \coloneq (\set{f(\bx_i, \by_u) + \epsilon < f(\bx_j, \by_u)} \cap \set{\text{PR}(u,i,j,\beta, k) = 1)}) \\
& \cup (\set{f(\bx_i, \by_u)  > f(\bx_j, \by_u)+ \epsilon} \cap \set{\text{PR}(u,i,j,\beta, k) = 0}).
\end{align*}

\begin{thm}
\label{pairwise_rank_levels_thm}
Let $\epsilon, \delta > 0$ and $\eta \in (0, \frac{\epsilon}{4})$. Suppose that $\sP_{\sG}$ is diverse and that almost every $y \in \sY$ is $(\frac{\epsilon}{4}, \delta)$-\discern . Let $r$ be a positive nondecreasing function such that $r(\frac{\epsilon}{4}) \geq \delta$ and $r(\eta) < \frac{\delta}{2}$. Suppose that almost every $y \in \sY$ is $r$-\diverse . Let $\frac{1}{2} > \alpha > \alpha^\prime > 0$. If $p \geq \max(n_1^{-\frac{1}{2}+ \alpha}, n_2^{-\frac{1}{2}+ \alpha})$, $n_1 p^2 \geq 16$, $n_1 \geq C_1 \log(n_2)^{\frac{1}{2 \alpha}}$ for some suitable universal constant $C_1$, and $n_2$ is sufficiently large, for all $u \in [n_2]$ and $i \neq j \in [n_1]$, the output of Pairwise-Rank with $k = n_2^{\alpha^\prime}$ and $\beta = \frac{p^2 n_1}{2}$ is such that
\begin{align*}
{\Pr}_{\set{\bx_i}, \set{\by_u}, \set{\ba_{u,l}}, \bOmega}(D^\epsilon_{u,i,j}) 
\leq &  2 \exp(- \frac{(n_2 - 1) p^2}{12}) + (n_2 - 1)\exp(-\frac{n_1 p^2 }{8}) +   2  \exp(-\gamma(\frac{\epsilon}{4}) k ) \\
+ & \frac{1}{ 1 - r(\frac{\epsilon}{2})} [3(n_2 - 1) p^2 \exp(- \frac{ \delta^2 n_1 p^2}{20}  )   \\
+ &  \exp([1  - \kappa(\frac{\epsilon}{2}) + \tau(\eta)+    \log(\frac{3(n_2 -1)p^2}{2})]k \\
- & k\log(k)  - \tau(\eta)\frac{(n_2 -1)p^2}{2})].
\end{align*}

\end{thm}

\begin{proof}[Proof of Theorem \ref{pairwise_rank_levels_thm}]
Fix $u \in [n_2]$, $i,j \in [n_1]$ such that $i \neq j$. Define:
\begin{align*}
W^{i,j}_u(\beta) & = \set{ v \in [n_2] : |N(u,v)| \geq \beta, (i,v), (j,v) \in \Omega}.
\end{align*}
Further, define the events:
\begin{align*}
A & = \set{|W^{i,j}_u(\beta)| \in [\frac{(n_2 - 1) p^2}{2}, \frac{3(n_2 - 1) p^2}{2}]}, \\
B & = \set{{\max}^{(k)}_{v \in W^{i,j}_u(\beta)} \rho^\prime(\by_u,\by_v) \geq 1 - \frac{\delta}{2}}, \\
C & = \set{ |R_{\bu \bv} - \nrho(\by_u, \by_v) | \leq \frac{\delta}{4}, \, \forall v \in W_u^{i,j}(\beta)} \\
E & = \set{|f(\bx_i, \by_u) - f(\bx_j, \by_u)| > \epsilon} \\
M & = \set{ \exists v \in W^{i,j}_u(\beta) \text{ s.t. } \rho^\prime(\by_u, \by_v) \geq 1 - \frac{\delta}{2} \text{ and } \exists l \in [L-1] \text{ s.t. } \ba_{v,l} \in (f(\bx_j, \by_v), f(\bx_i, \by_v))}
\end{align*}
By several applications of the law of total probability, we have that
\begin{align}
\Pr( D_{u,i,j}^\epsilon) \leq & \Pr(D_{u,i,j}^\epsilon | E) + \Pr(D_{u,i,j}^\epsilon | E^c)\nonumber \\
= & \Pr(D_{u,i,j}^\epsilon | E)  \nonumber \\
\leq & \Pr(D_{u,i,j}^\epsilon | A, B, C, M, E) + \Pr( A^c | E)  + \Pr( B^c | A, E) \nonumber \\
+ & \Pr( C^c | A, B, E) + \Pr(M^c | A, B, C, E) \nonumber \\
= & \Pr(D_{u,i,j}^\epsilon | A, B, C, M, E) + \Pr( A^c ) + \Pr( B^c | A, E ) \label{A_E_independence_discrete_theorem} \\
+ &  \Pr( C^c | A, B, E) + \Pr(M^c | A, B, C, E) \nonumber 
\end{align}
Line \eqref{A_E_independence_discrete_theorem} follows from the independence of $\bOmega$ from $\set{\bx_s}_{s \in [n_1]}$ and $\set{\by_v}_{v \in [n_2]}$. We will bound each term in the above upper bound. By Lemma \ref{error_bound_discrete}, 
\begin{align}
\Pr(D_{u,i,j}^\epsilon| A, B, C, M, E) = 0. \label{D_discrete_theorem}
\end{align}

Setting $\lambda = \frac{1}{2}$ in Lemma \ref{sufficient_overlap} yields that
\begin{align}
\Pr(A^c) & = \Pr(|W_u^{i,j}(\beta)| \not \in [\frac{(n_2 - 1) p^2}{2}, \frac{3(n_2 - 1) p^2}{2}]) \nonumber \\
& \leq 2 \exp(- \frac{(n_2 - 1) p^2}{12}) + (n_2 - 1)\exp(-\frac{n_1 p^2 }{8}). \label{A_complement_discrete_theorem}
\end{align}

Next, we bound $\Pr( B^c | A, E )$. By Bayes theorem, 
\begin{align}
\Pr( B^c | A, E ) & \leq \frac{\Pr(B^c | A)}{\Pr(E |A)} \nonumber  \\
& = \frac{\Pr(B^c | A)}{\Pr(E)} \label{A_E_2_independence_discrete_theorem} \\
& <  \frac{\Pr(B^c | A)}{1-r(\frac{\epsilon}{2})}. \label{B_complement_bayes_discrete_theorem}
\end{align}
Line \eqref{A_E_2_independence_discrete_theorem} follows from the independence of $\bOmega$ from $\set{\bx_s}_{s \in [n_1]}$ and $\set{\by_v}_{v \in [n_2]}$. Line \eqref{B_complement_bayes_discrete_theorem} follows since by hypothesis almost every $y \in \sY$ is $r$-\diverse .

Since almost every $y \in \sY$ is $(\frac{\epsilon}{4}, \delta)$-\discern \, and $r$-\diverse , and $\eta > 0$ is such that $r(\eta) < \frac{\delta}{2}$, Lemma \ref{good_neighbor_k} yields that
\begin{align}
\Pr(B^c | A) & = \Pr({\max}^{(k)}_{v \in W_u^{i,j}(\beta)} \rho^\prime(\by_u,\by_v) < 1 - \frac{\delta}{2} | A) \nonumber \\
& \leq \Pr( {\max}^{(k)}_{v \in W_u^{i,j}(\beta)} \rho^\prime(\by_u, \by_v) < 1 - r(\eta) | A) \nonumber \\
&  \leq  \exp((1  - \kappa(\frac{\epsilon}{4}) + \tau(\eta)+  \log(3 \frac{(n_2 -1)p^2}{2}))k - k\log(k)  - \tau(\eta)\frac{(n_2 -1)p^2}{2})).\label{B_complement_discrete_theorem}
\end{align}

Next, we bound $\Pr(C^c | A,B,E)$. By Bayes theorem,
\begin{align*}
\Pr(C^c | A,B,E) & \leq \frac{\Pr(C^c | A,B)}{\Pr(E|A,B)} .
\end{align*}
Fix $\by_u = y_u$ $r$-\diverse  \, such that $A$ and $B$ occur. Then, since $\set{\bx_s}_{s \in [n_1]}$, $\set{\by_v}_{v \in [n_2]}$, and $\bOmega$ are independent and $ y_u$ is $r$-\diverse,
\begin{align*}
& {\Pr}_{\set{\by_v}_{v \in [n_2]}, \set{\bx_s}_{s \in [n_1]}, \bOmega} (|f(\bx_i, y_u) - f(\bx_j, y_u)| > \epsilon |  \by_u = y_u) \\
& \qquad =  {\Pr}_{\bx_i, \bx_j}(|f(\bx_i, y_u) - f(\bx_j, y_u)| > \epsilon |  \by_u = y_u) \\
& \qquad ={\Pr}_{\bx_i, \bx_j}(|f(\bx_i, y_u) - f(\bx_j, y_u)| > \epsilon )  > 1 - r(\frac{\epsilon}{2}) .
\end{align*}
Since the above bound holds for all $y_u$ such that $A \cap B$ holds, taking the expectation of the above bound with respect to $\by_u$ over the set $A \cap B$ gives
\begin{align*}
\Pr(|f(\bx_i, \by_u) - f(\bx_j, \by_u)| > \epsilon | A,B) > 1 - r(\frac{\epsilon}{2}) .
\end{align*}
Thus, 
\begin{align}
\Pr(C^c | A,B,E) < \frac{\Pr(C^c | A,B)}{ 1 - r(\frac{\epsilon}{2})}. \label{C_complement_bayes_discrete_theorem}
\end{align}

Since by hypothesis $\frac{1}{2} > \alpha > \alpha^\prime > 0$, $p \geq \max(n_1^{-\frac{1}{2}+ \alpha}, n_2^{-\frac{1}{2}+ \alpha})$ and $k = n_2^{\alpha^\prime}$, if $n_2$ is sufficiently large, the bound in line \eqref{B_complement_discrete_theorem} is less than $\frac{1}{2}$. Then, by Bayes rule, the union bound, and Lemma \ref{R_concentration_discrete}, 
\begin{align}
\Pr(C^c | A, B) & \leq  \frac{\Pr(C^c |A)}{\Pr(B|A)} \leq 2 \Pr(C^c |A) \nonumber \\
& = 2 \Pr( \exists v \in W_u^{i,j}(\beta), |R_{\bu \bv} - \nrho(\by_u , \by_v) | > \frac{\delta}{4} | A) \nonumber \\
& \leq 3(n_2 - 1) p^2  \Pr(|R_{\bu \bv} - \nrho(\by_u , \by_v) | > \frac{\delta}{4} | A) \label{union_bound_C_complement_discrete_theorem} \\
& \leq 3(n_2 - 1) p^2 \exp(- \frac{\delta^2}{4} \left \lfloor\frac{\beta}{2} \right \rfloor ) \nonumber \\
& = 3(n_2 - 1) p^2 \exp(- \frac{\delta^2}{4} \left \lfloor\frac{n_1 p^2}{4} \right \rfloor ) \nonumber \\
& \leq 3(n_2 - 1) p^2 \exp(- \frac{ \delta^2 n_1 p^2}{20}  )  \label{C_complement_discrete_theorem}
\end{align}
where line \eqref{union_bound_C_complement_discrete_theorem} follows by the union bound and line \eqref{C_complement_discrete_theorem} follows because $n_1 p^2 \geq 16$ and $\forall x \geq 15$, $\left \lfloor \frac{x}{4} \right \rfloor \geq \frac{x}{5}$. 

Since by hypothesis $\frac{1}{2} > \alpha > 0$, $p \geq \max(n_1^{-\frac{1}{2}+ \alpha}, n_2^{-\frac{1}{2}+ \alpha})$, and $n_1 \geq C_1 \log(n_2)^{\frac{1}{2 \alpha}}$ for some constant $C_1$, if $n_2$ is sufficiently large, the bound in line \eqref{C_complement_bayes_discrete_theorem} is eventually less than $\frac{1}{2}$. Thus, using Bayes rule and Lemma \ref{monotonic_function_lemma},
\begin{align}
\Pr(M^c | A,B,C,E) & \leq \frac{\Pr(M^c | A, B,E)}{\Pr(C | A, B,E)} \nonumber \\
& \leq 2 \Pr(M^c | A, B,E) \nonumber \\
& \leq 2  \exp(-\gamma(\frac{\epsilon}{4}) k ). \label{M_complement_discrete_theorem}
\end{align}

Putting together lines \eqref{A_E_independence_discrete_theorem}, \eqref{D_discrete_theorem}, \eqref{A_complement_discrete_theorem}, \eqref{B_complement_bayes_discrete_theorem}, \eqref{B_complement_discrete_theorem}, \eqref{C_complement_bayes_discrete_theorem}, \eqref{C_complement_discrete_theorem}, and \eqref{M_complement_discrete_theorem} we have
\begin{align*}
 \Pr(D_{u,i,j}^\epsilon) \leq &   2 \exp(- \frac{(n_2 - 1) p^2}{12}) + (n_2 - 1)\exp(-\frac{n_1 p^2 }{8})  +  2  \exp(-\gamma(\frac{\epsilon}{4}) k ) \\
  + & \frac{1}{ 1 - r(\frac{\epsilon}{2})} [ 3(n_2 - 1) p^2 \exp(- \frac{ \delta^2 n_1 p^2}{20}  )   \\
+ &  \exp([1  - \kappa(\frac{\epsilon}{4}) + \tau(\eta)+  \log(3\frac{(n_2 -1)p^2}{2})]k - k\log(k)  - \tau(\eta)\frac{(n_2 -1)p^2}{2})].
\end{align*}
\end{proof}

\begin{proof}[Proof of Theorem \ref{multi_rank_discrete_thm}] 
The proof follows the same steps as the proof of Theorem \ref{multi_rank_thm}, but applies Theorem \ref{pairwise_rank_levels_thm} instead of Theorem \ref{pairwise_rank_thm}.
\end{proof}

\begin{proof}[Proof of Corollary \ref{pairwise_rank_cor_discrete}]
The only new term that did not appear in Corollary \ref{pairwise_rank_cor_discrete} is, ignoring constants, of the form
\begin{align*}
n_2^2 n_1^2 \exp(\log(n_2p^2) k - k \log(k) - n_2 p^2). 
\end{align*}
Using $\alpha > \alpha^\prime$ and $n_1 \leq C_1 n_2$, as $n_2 \longrightarrow \infty$,
\begin{align*}
n_2^2 n_1^2 \exp(\log(n_2p^2) k & - k \log(k) - n_2 p^2) \\
& \leq \exp(2 \log(n_2) + 2 \log(n_1) + \log(n_2^{2 \alpha}) n_2^{\alpha^\prime} - n_2^{\alpha^\prime} \log(n_2) \alpha^\prime - n_2^{2 \alpha}) \\
& \leq \exp((2 + 2 C_1) \log(n_2) + (2 \alpha - \alpha^\prime) \log(n_2) n_2^{\alpha^\prime} - n_2^{2 \alpha}) \\
& \longrightarrow 0
\end{align*}
\end{proof}

\section{Technical Lemmas}
\label{technical_lemmas_section}

We separate the lemmas into three sections: lemmas for both the continuous and discrete rating settings, lemmas for the continuous rating setting, and lemmas for the discrete rating setting.

\subsection{Lemmas Common to the Continuous Rating Setting and the Discrete Rating Setting}

Lemma \ref{sufficient_overlap} establishes that for a user $u \in [n_2]$ and distinct items $i, j \in [n_1]$, with high probability there are many other users that have rated items $i$ and $j$ and many items in common with user $u$. It is similar to Lemma 1 from \citet{lee2016}.

\begin{lemma}
\label{sufficient_overlap}
Fix $u \in [n_2]$, $i \neq j \in [n_1]$, and let $\lambda > 0$ and $2 \leq \beta \leq  \frac{n_1 p^2}{2}$. Let $W^{i,j}_u(\beta) = \set{ v \in [n_2] : |N(u,v)| \geq \beta, (i,v), (j,v) \in \Omega}$. Then, 
\begin{align*}
& {\Pr}_{\bOmega}(|W^{i,j}_u(\beta)| \not \in [(1 - \lambda)(n_2 - 1) p^2,(1 - \lambda)(n_2 - 1) p^2]) \\
&\qquad \leq  2 \exp(- \frac{\lambda^2 (n_2 - 1) p^2}{3}) + (n_2 - 1)\exp(-\frac{n_1 p^2 }{8}).
\end{align*}
\end{lemma}

\begin{proof}
Define the following binary variables for all $v \in [n_2] \setminus \set{u}$. $E_v = 1$ if $|N(u,v)| \geq \beta$ and $0$ otherwise, $F_v = 1$ if $(i,v) \in \bOmega$ and $0$ otherwise, and $G_v = 1$ if $(j,v) \in \bOmega$ and $0$ otherwise. Observe that $|W_u^{i,j}(\beta)| = \sum_{v \neq u} E_v F_v G_v$. Fix $0 \leq a < b \leq n_2 -1$. Observe that if $\sum_{v \neq u} F_v G_v \in [a,b]$ and $\sum_{v \neq u} E_v = n_2 -1$, then $|W_u^{i,j}(\beta)| \in [a,b]$. Thus, the contrapositive implies that for any $0 \leq a < b \leq n_2 - 1$,
\begin{align*}
{\Pr}_{\bOmega}(|W_u^{i,j}(\beta)| \not \in [a,b]) & \leq {\Pr}_{\bOmega}(\sum_{v \neq u} F_v G_v \not \in [a,b] \cup \sum_{v \neq u} E_v < n_2 -1) \\
& \leq {\Pr}_{\bOmega}(\sum_{v \neq u} F_v G_v \not \in [a,b]) + {\Pr}_{\bOmega}(\sum_{v \neq u} E_v < n_2 -1).
\end{align*}
$\sum_{v \neq u} F_v G_v $ is a binomial random variable with parameters $n_2-1$ and $p^2$. Letting $a = (1-\lambda)(n_2 -1)p^2$ and $b = (1+\lambda)(n_2-1)p^2$, Chernoff's multiplicative bound (Proposition \ref{chernoff_mult}) yields that
\begin{align*}
{\Pr}_{\bOmega}(\sum_{v \neq u} F_v G_v  \not \in [(1-\lambda)(n_2 -1) p^2, (1+\lambda)(n_2-1)p^2]) \leq 2 \exp(-\frac{\lambda^2(n_2-1) p^2}{3}).
\end{align*}
Since $N(u,v)$ is binomial with parameters $n_1$ and $p^2$, by Chernoff's multiplicative bound (Proposition \ref{chernoff_mult}),
\begin{align*}
{\Pr}_{\bOmega}(E_v = 0) & = {\Pr}_{\bOmega}(N(u,v) \leq \beta) \\ 
& \leq {\Pr}_{\bOmega}(N(u,v) \leq \frac{n_1 p^2}{2}) \\
& \leq \exp(- \frac{n_1 p^2}{8}).
\end{align*}
Then, by the union bound,
\begin{align*}
{\Pr}_{\bOmega}(\sum_{v \neq u} E_v < n_2 -1) & = {\Pr}_{\bOmega}(\exists v \in [n_2] \setminus \set{u} : E_v = 0)  \\
& \leq (n_2 -1 )\exp(- \frac{n_1 p^2}{8}).
\end{align*}
\end{proof}

To convert the pairwise comparisons to a ranking, we use the Copeland ranking procedure (Algorithm \ref{Copeland} in the main document). Lemma \ref{copeland_lemma} establishes that if the output of the Pairwise-Rank algorithm is such that for all $i,j \in [n_1]$ and $u \in [n_2]$, $D_{u,i,j}^\epsilon$ does not occur, then applying the Copeland ranking procedure to $A$ (as defined in Multi-Rank) yields a $\widehat{\sigma}$ such that ${\dis}_{2\epsilon}(\widehat{\sigma}, H) = 0$.

\begin{lemma}
\label{copeland_lemma}
Let $\epsilon > 0$, $u \in [n_2]$, $A$ as defined in Multi-Rank (Algorithm \ref{multi_rank_algorithm}), and $\widehat{\sigma}_u = \text{Copeland}(A_{u,:,:})$. If for all $i \neq j \in [n_1]$ $f(x_i, y_u) > f(x_j, y_u) + \epsilon$ implies that $A_{u,i,j} = 1$, then for all $i \neq j \in [n_1]$ $h_u(x_i,y_u) > h_u(x_j, y_u)$ and $f(x_i, y_u) > f(x_j, y_u) + 2 \epsilon$ implies that $\widehat{\sigma}_u(i) > \widehat{\sigma}_u(j)$.
\end{lemma}

\begin{proof}
Let $i \neq j \in [n_1]$ such that $h_u(x_i,y_u) > h_u(x_j, y_u)$ and $f(x_i, y_u) > f(x_j, y_u) + 2 \epsilon$. Let $l \in [n_1]$ such that $l \neq i$ and $l \neq j$. We claim that if $A_{u,i,l} = 0$, then $A_{u,j,l} = 0$. If $A_{u,i,l} = 0$, then by the hypothesis $f(x_i,y_u) \leq f(x_l, y_u) + \epsilon$.  Then,
\begin{align*}
f(x_j, y_u) + 2\epsilon < f(x_i, y_u) \leq f(x_l,y_u) + \epsilon
\end{align*}
so that $f(x_j, y_u) + \epsilon < f(x_l,y_u)$. Then, by the hypothesis, $A_{u,j,l} = 0$, establishing the claim.

The contrapositive of the claim is that if $A_{u,j,l} = 1$, then $A_{u,i,l} = 1$. Then, 
\begin{align*}
I_j = \sum_{l = 1, l \neq j}^{n_1} A_{u,j,l} = \sum_{l = 1, l \not \in \set{ j, i}}^{n_1} A_{u,j,l}  \leq \sum_{l = 1, l \not \in \set{ j, i}}^{n_1} A_{u,i,l} = I_i - 1 < I_i
\end{align*}
so that $\widehat{\sigma}_u(i) > \widehat{\sigma}_u(j)$.
\end{proof}

Recall the definition of our problem-specific constants: $\tau(\epsilon) = \inf_{y_0 \in \sY}\Pr_{\by_u}(d_{\sY}(y_0, \by_u) \leq \epsilon)$, $\kappa(\epsilon) = \inf_{y_0 \in \sY} \Pr_{\by_u}(d_{\sY}(y_0, \by_u) > \epsilon)$, and $\gamma(\epsilon) = \inf_{z \in [-N,N]} \sP_{\set{\ba_{u,l}}_{l \in [L-1]}}(\exists l \in [L-1] :  d_\bbR(z,\ba_{u,l}) \leq \epsilon)$. Lemma \ref{kappa_tau_lemma} establishes that under our assumptions, for all $\epsilon > 0$, $\tau(\epsilon) > 0$, $\kappa(\epsilon) < 1$, and $\gamma(\epsilon) > 0$.

\begin{lemma}
\label{kappa_tau_lemma}
If there exists $\epsilon > 0$ such that $\tau(\epsilon) = 0$, or $\kappa(\epsilon) = 1$, then there exists a point $z \in \sY$ such that $\sP_{\sY}(B_\epsilon(z)) = 0$. Similarly, if there exists $\epsilon > 0$ such that $\gamma(\epsilon) = 0$, then there exists $z \in [-N,N]$ such that $\sP_l(B_\epsilon(z)) = 0$ for all $l \in [L-1]$.
\end{lemma}

\begin{proof}
Let $\epsilon > 0$ and suppose $\tau(\epsilon) = 0$. Then, there exists a sequence of points $z_1, z_2, \ldots \in \sY$ such that for every $n$, $\sP_{\sY}(B_\epsilon(z_n)) \leq \frac{1}{n}$. Since $\sY$ is compact by assumption, there exists a convergent subsequence $z_{i_1}, z_{i_2}, \ldots$ to $z$. 

We claim that for all $z^\prime \in \sY$, there exists a sufficiently large $N$ such that $z^\prime \in B_\epsilon(z_{i_N})$ if and only if $z^\prime \in B_\epsilon(z)$. Fix $z^\prime \in B_\epsilon(z)$. Since $B_\epsilon(z)$ is open, there exists $\delta > 0$ such that $d_{\sY}(z,z^\prime) < \delta < \epsilon$. Let $N$ large enough such that $d_{\sY}(z, z_{i_N}) \leq \epsilon - \delta$. Then, by the triangle inequality,
\begin{align*}
d(z^\prime, z_{i_N}) \leq d(z^\prime, z) + d(z_{i_N}, z) \leq \delta + \epsilon - \delta = \epsilon
\end{align*}
so that $z^\prime \in B_\epsilon(z_{i_N})$. A similar argument shows the other direction of the claim. Since a probability space has finite measure, by the dominated convergence theorem,
\begin{align*}
\sP_{\sY}(B_\epsilon(z)) & = \lim_{n \longrightarrow \infty} \sP_{\sY}( B_\epsilon(z_{i_n})) \leq \lim_{n \longrightarrow \infty} \frac{1}{i_n} = 0.
\end{align*}

Next, suppose $\kappa(\epsilon) = 1$. Then, there exists a sequence of points $z_1, z_2, \ldots \in \sY$ such that for every $n$, $\sP_{\sY}(B_\epsilon(z_n)^c) \geq 1 - \frac{1}{n}$. Then, for every $n$, $\sP_{\sY}(B_\epsilon(z_n)) \leq \frac{1}{n}$   A similar argument from the $\tau(\cdot)$ case using the dominated convergence theorem shows that $\sP_{\sY}(B_\epsilon(z)) = 0$.

Since $[-N,N]$ is compact and $\gamma$ has a similar definition to $\tau$, the result for $\gamma(\cdot)$ follows by an argument similar to the one used for the $\tau(\cdot)$ case.
\end{proof}

%

%

\subsection{Lemmas for Continuous Rating Setting}

Lemma \ref{discerning_lemma_continuous} uses the notion of $r$-\diverse  \, to relate the distance between points in $\sY$ and to a lower bound on $\rho(y_u,y_v)$.

\begin{lemma}
\label{discerning_lemma_continuous}
Let $r$ be a positive nondecreasing function. If $y_u \in \sY$ is $r$-\diverse , then for any $\epsilon > 0$, if $y_v \in B_\epsilon(y_u)$, then $\rho(y_u,y_v) > 1 - r(\epsilon)$.
\end{lemma}

\begin{proof}
Suppose that $d(y_u, y_v) \leq \epsilon$. Suppose that $\bx_i = x_i$ and $\bx_j = x_j$ such that $|f(x_i, y_u) - f(x_j, y_u)| > 2 \epsilon$ and without loss of generality suppose that $h_u(x_i, y_u) \geq h_u(x_j, y_u)$. Then, since $f$ is Lipschitz,
\begin{align*}
f(x_i, y_v) \geq f(x_i, y_u) - \epsilon > f(x_j, y_u) + \epsilon \geq f(x_j, y_v).
\end{align*}
Hence, $h_{v}(x_i, y_v) \geq h_v(x_j, y_v)$. Thus,
\begin{align*}
\rho(y_u, y_v) \geq {\Pr}_{\bx_i, \bx_j}(|f(\bx_i, y_u) - f(\bx_j, y_u)| > 2 \epsilon) > 1 - r(\epsilon),
\end{align*}
where the last inequality follows from the hypothesis that $y_u$ is $r$-\diverse . Thus, we conclude the result.
\end{proof}

Lemma \ref{good_neighbor} establishes that if $S \subset [n_2] \setminus \set{u}$ is a large enough set, then with high probability there is at least one element $\by_v$ in $S$ that tends to agree with $\by_u$. 

\begin{lemma}
\label{good_neighbor}
Let $r$ be a positive non-decreasing function and suppose that almost every $y \in \sY$ is $r$-\diverse . Let $S \subset [n_2] \setminus \set{u}$. Then, $\forall \epsilon > 0$,
\begin{align*}
{\Pr}_{\by_v, \by_u}( \max_{v \in [S]} \rho(\by_v, \by_u) \leq 1 - r(\epsilon)) \leq [1 - \tau(\epsilon)]^{|S|}.
\end{align*}
\end{lemma}

\begin{proof}
Fix $\by_u = y_u \in \sY$ that is $r$-\diverse .  By Lemma \ref{discerning_lemma_continuous}, if $\by_v = y_v$ is such that $d(y_u, y_v) \leq \epsilon$, then $\rho(y_u, y_v) > 1 - r(\epsilon)$. Hence,
\begin{align*}
{\Pr}_{\by_v}(d(y_u, \by_v)) \leq \epsilon) \leq {\Pr}_{\by_v}( \rho(y_u,\by_v) > 1-r(\epsilon)). 
\end{align*}
Then, 
\begin{align*}
{\Pr}_{\by_v}( \rho(y_u,\by_v) \leq 1-r(\epsilon)) \leq {\Pr}_{\by_v}(d(y_u, \by_v)) > \epsilon) = 1 - {\Pr}_{\by_v}(d(y_u, \by_v)) \leq \epsilon) \leq 1 - \tau(\epsilon).
\end{align*}
The RHS does not depend on $y_u$, and $\by_v, \by_u$ are independent and almost every $y \in \sY$ is $r$-\diverse, so we can take the expectation with respect to $\by_u$ to obtain
\begin{align}
{\Pr}_{\by_v, \by_u}(\rho(\by_v, \by_u) \leq 1 - r(\epsilon)) \leq 1 - \tau(\epsilon). \label{rho_bound_good_neighbor}
\end{align}

Finally,
\begin{align*}
{\Pr}_{\set{\by_v}_{v \in S}, \by_u}( \max_{v \in [S]} \rho(\by_v, \by_u) \leq 1 - r(\epsilon)) & = {\Pr}_{\by_v, \by_u}(\rho(\by_v, \by_u) \leq 1 - r(\epsilon))^{|S|} \\
& \leq [1 - \tau(\epsilon)]^{|S|},
\end{align*}
where the first equality follows from the independence of $\by_1, \ldots, \by_{n_2}$ and the inequality follows from line \eqref{rho_bound_good_neighbor}.
\end{proof}

Lemma \ref{R_concentration} establishes that $R_{\bu, \bv}$ concentrates around $\rho(\by_u, \by_v)$.

\begin{lemma}
\label{R_concentration}
Let $u \neq v \in [n_2]$, $i \neq j \in [n_1]$, $\eta > 0$, $\beta \geq 2$, and $W^{i,j}_u(\beta)$ be defined as in Lemma \ref{sufficient_overlap}. Then,
\begin{align*}
{\Pr}(|R_{\bu, \bv} - \rho(\by_u, \by_v)| > \frac{\eta}{4} | v \in W^{i,j}_u(\beta)) \leq 2 \exp(- \frac{\eta^2}{4} \left \lfloor\frac{\beta}{2} \right \rfloor ).
\end{align*}
\end{lemma}

\begin{proof}
Fix $\by_u = y_u$ and $\by_v = y_v$. Recall that if $I(u,v) \neq \emptyset$, then
\begin{align*}
R_{u,v} = \frac{1}{ |I(u,v)| } \sum_{(s,t) \in I(u,v)} \ind{(h_u(\bx_s,y_u) - h_u(\bx_t, y_u)) (h_v(\bx_s, y_v) - h_v(\bx_t, y_v)) \geq 0}.
\end{align*}
Since $I(u,v)$ consists of pairs of indices that do not overlap, conditioned on $\by_u = y_u, \by_v = y_v$, and any nonempty $I(u,v)$, $\set{ \ind{(h_u(\bx_s,y_u) - h_u(\bx_t, y_u)) (h_v(\bx_s, y_v) - h_v(\bx_t, y_v)) \geq 0}:  (s,t) \in I(u,v)}$ is a set of independent random variables. Further, each has mean $\rho(y_u,y_v)$. Thus, by Chernoff's bound (Proposition \ref{hoeffding}),
\begin{align*}
{\Pr}(|R_{u, v} - \rho(y_u, y_v)| > \frac{\eta}{4} | \by_u = y_u, \by_v = y_v, I(u,v)) \leq \exp(-\frac{\eta^2}{2} |I(u, v)|)
\end{align*}
When $v \in v \in W^{i,j}_u(\beta)$, $|I(u,v)| \geq \left \lfloor\frac{\beta}{2} \right \rfloor $. Since the above bound holds for all $y_u, y_v$, it follows that
\begin{align*}
{\Pr}(|R_{\bu, \bv} - \rho(\by_u, \by_v)| > \frac{\eta}{4} | v \in W^{i,j}_u(\beta)) \leq 2 \exp(- \frac{\eta^2}{4} \left \lfloor\frac{\beta}{2} \right \rfloor ).
\end{align*}
\end{proof}

Lemma \ref{error_bound_continuous} establishes that conditional on $A,B,C$ (defined in the proof of Theorem \ref{pairwise_rank_thm}), the event $D_{u,i,j}^\epsilon $ does not occur with probability $1$.

\begin{lemma}
\label{error_bound_continuous}
Under the setting described in Theorem \ref{pairwise_rank_thm}, let $u \in [n_2]$ and $i \neq j \in [n_1]$. Then, $\Pr(D_{u,i,j}^\epsilon | A,B,C) = 0$. 
\end{lemma}

\begin{proof}
Define the events
\begin{align*}
E_1 & =  \set{f(\bx_i, \by_u) + \epsilon < f(\bx_j, \by_u)} \\
E_2 & =  \set{f(\bx_i, \by_u)  > f(\bx_j, \by_u)+ \epsilon}
\end{align*}
By the union bound and law of total probability,
\begin{align*}
\Pr(D_{u,i,j}^\epsilon | A,B,C) & \leq \Pr(\text{PR}(u,i,j,\beta,k) = 1  \cap E_1 | A,B,C) \\
& + \Pr(\text{PR}(u,i,j,\beta,k) = 0  \cap E_2 | A,B,C) \\
& \leq  \Pr(\text{PR}(u,i,j,\beta,k) = 1  | A,B,C, E_1) \\
& + \Pr(\text{PR}(u,i,j,\beta,k) = 0 | A,B,C, E_2).
\end{align*}
The argument for bounding each of these is similar and, thus, we bound the term $ \Pr(\text{PR}(u,i,j,\beta,k) = 1  | A,B,C, E_1)$.

Fix $\set{\by_v = y_v}_{v \in [n_2]}$ $r$-\diverse \, and $(\frac{\epsilon}{2}, \delta)$-\discern , $\set{\bx_s = x_s}_{s \in [n_1]}$, and $\bOmega = \Omega$ such that the event $A \cap B \cap C \cap E_1$ occurs. We claim that Pairwise-Rank puts $V = \set{v}$ (see Algorithm \ref{pairwise_rank} for definition of $V$) such that $y_v \in B_{\frac{\epsilon}{2}}(y_u)$. On the event $B $, there is $v \in W_u^{i,j}(\beta)$ with $\rho(y_u, y_v) \geq 1 - \frac{\delta}{2}$. Since $y_u$ is $(\frac{\epsilon}{2}, \delta)$-\discern , it follows that $y_v \in B_{\frac{\epsilon}{2}}(y_u)$. Suppose that $w \in W_u^{i,j}(\beta)$ such that $y_{w} \in B_{\frac{\epsilon}{2}}(y_u)^c$. Since $y_u$ is $(\frac{\epsilon}{2}, \delta)$-\discern, $\rho(y_w, y_u) < 1 - \delta$. Then,
\begin{align}
R_{w,u} & \leq \rho(y_w,y_u) + \frac{\delta}{4} \label{apply_C_1_error_bound_continuous} \\
&  < 1 - \frac{3}{4} \delta \nonumber \\
& \leq \rho(y_u, y_v) - \frac{\delta}{4} \nonumber  \\
& \leq R_{u,v} \label{apply_C_2_error_bound_continuous}
\end{align}
where lines \eqref{apply_C_1_error_bound_continuous} and \eqref{apply_C_2_error_bound_continuous} follow by event $C$ and $v,w \in W_u^{i,j}(\beta)$. Thus, the claim follows. Conditional on $E_1$, we have that $f(x_i, y_u) + \epsilon < f(x_j, y_u)$. Then, using the Lipschitzness of $f$,
\begin{align*}
f(x_i, y_v) \leq f(x_i, y_u) + \frac{\epsilon}{2} < f(x_j, y_u) - \frac{\epsilon}{2} \leq f(x_j, y_v).
\end{align*}
Since $g_v$ is strictly increasing by hypothesis, $h_v(x_i, y_v) < h_v(x_j, y_v)$. Thus, Pairwise-Rank with $k=1$ outputs $0$. Consequently, 
\begin{align*}
\Pr(\text{PR}(u,i,j,\beta,k) = 1  | A,B,C, E_1, \set{\by_v = y_v}_{v \in [n_2]} \set{\bx_s = x_s}_{s \in [n_1]}, \bOmega = \Omega) = 0
\end{align*}
Since almost every $y \in \sY$ is $r$-\diverse \, and $(\frac{\epsilon}{2}, \delta)$-\discern , taking the expectation wrt $\set{\by_v}_{v \in [n_2]}$, $\set{\bx_s }_{s \in [n_1]}$, $\bOmega$ on the set $A \cap B \cap C \cap E_1$ of the last equality gives the result.
\end{proof}

\subsection{Lemmas for Discrete Rating Setting}

Lemma \ref{discerning_lemma_discrete} is the analogoue of Lemma \ref{discerning_lemma_continuous} for the discrete case. The proof is very similar.

\begin{lemma}
\label{discerning_lemma_discrete}
Let $r$ be a positive non-decreasing function. If $y_u \in \sY$ is $r$-\diverse , then for any $\epsilon > 0$, if $y_v \in B_\epsilon(y_u)$, then $\rho^\prime(y_u,y_v) > 1 - r(\epsilon)$.
\end{lemma}

\begin{proof}
Suppose $y_v$ is such that $d(y_u, y_v) \leq \epsilon$. We claim that under this assumption
\begin{align}
\rho^\prime(y_u, y_v) \geq {\Pr}_{\bx_i, \bx_j }(|f(\bx_i, y_u) - f(\bx_j, y_u)| > 2 \epsilon). \label{claim_line_metric_lemma_discrete}
\end{align}
Fix $\bg_u = g_u$ and $\bg_v = g_v$, and $\bx_i = x_i$ and $\bx_j = x_j$ such that $|f(x_i, y_u) - f(x_j, y_u)| > 2 \epsilon$. Without loss of generality, suppose that $h_u(x_i, y_u) \geq h_u(x_j, y_u)$. Then, since $f$ is Lipschitz,
\begin{align*}
f(x_i, y_v) \geq f(x_i, y_u) - \epsilon > f(x_j, y_u) + \epsilon \geq f(x_j, y_v).
\end{align*}
Hence, $h_{v}(x_i, y_v) \geq h_v(x_j, y_v)$, establishing that
\begin{align}
& \rho^\prime(y_u, y_v | \bg_u = g_u, \bg_v = g_v) \nonumber \\
& \qquad =  {\Pr}_{\bx_i, \bx_j}([g_u(f(\bx_i, y_u)) - g_u(f(\bx_j, y_u))][g_v(f(\bx_i, y_v)) - g_u(f(\bx_j, y_v))] \geq 0) \nonumber \\
& \qquad \geq  {\Pr}_{\bx_i, \bx_j}(|f(\bx_i, y_u) - f(\bx_j, y_u)| > 2 \epsilon). \label{rho_line_metric_lemma_discrete}
\end{align}
Since $\set{\bg_u, \bg_v, \bx_i, \bx_j}$ are independent, taking the expectation with respect to $\bg_u$ and $\bg_v$ in line \eqref{rho_line_metric_lemma_discrete} establishes line \eqref{claim_line_metric_lemma_discrete}.
Thus,
\begin{align*}
\rho^\prime(y_u, y_v) \geq {\Pr}_{\bx_i, \bx_j}(|f(\bx_i, y_u) - f(\bx_j, y_u)| > 2 \epsilon) > 1 - r(\epsilon),
\end{align*}
where the last inequality follows from the hypothesis that $y_u$ is $r$-\diverse .
\end{proof}

Lemma \ref{good_neighbor_k} is the analogoue of Lemma \ref{good_neighbor} for the discrete case.
\begin{lemma}
\label{good_neighbor_k}
Let $\epsilon, \delta > 0$. Let $r$ be a positive nondecreasing function such that $r(\epsilon) \geq \delta$ and $r(\eta) < \delta$ for some $\eta > 0$. Suppose that almost every $y \in \sY$ is $(\epsilon, \delta)$-\discern \, and $r$-\diverse . Let $R_2 \geq R_1 \geq 0$ be constants. Then, for any $S \subset [n_2]$ depending on $\bOmega$ and $k \leq R_1$,
\begin{align*}
& {\Pr}_{\by_v, \by_u}( {\max}^{(k)}_{v \in [S]} \rho^\prime(\by_v, \by_u) \leq 1 - r(\eta)\, | R_1 \leq |S| \leq R_2 ) \\
& \qquad \leq  \exp((1  - \kappa(\epsilon) + \tau(\eta)+ \log(R_2))k - k\log(k)  - \tau(\eta)R_1) | R_1 \leq |S| \leq R_2).
\end{align*}
\end{lemma}

\begin{proof}
Let $C_{\eta} = {\Pr}_{\by_v, \by_u}(\rho^\prime(\by_v, \by_u) \leq 1 - r(\eta))$. 

\textbf{Claim: } $C_\eta  \leq 1 - \tau(\eta)$. 

Fix $\by_u = y_u \in \sY$ $r$-\diverse .  By Lemma \ref{discerning_lemma_discrete}, if $\by_v = y_v$ is such that $d(y_u, y_v) \leq \epsilon$, then $\rho^\prime(y_u, y_v) > 1 - r(\epsilon)$. Hence,
\begin{align*}
{\Pr}_{\by_v}(d(y_u, \by_v)) \leq \epsilon) \leq {\Pr}_{\by_v}( \rho^\prime(y_u,\by_v) > 1-r(\epsilon)). 
\end{align*}
Then, 
\begin{align*}
{\Pr}_{\by_v}( \rho^\prime(y_u,\by_v) \leq 1-r(\epsilon)) \leq {\Pr}_{\by_v }(d(y_u, \by_v)) > \epsilon) = 1 - {\Pr}_{\by_v }(d(y_u, \by_v)) \leq \epsilon) \leq 1 - \tau(\epsilon),
\end{align*}
where the last inequality follows by the definition of $\tau(\cdot)$. The RHS does not depend on $y_u$, and $\by_v, \by_u$ are independent, so we can take the expectation with respect to $\by_u$ to establish the claim.

\textbf{Claim: } $1 - C_\eta \leq 1 - \kappa(\epsilon)$. 

Since almost every $y \in \sY$ is $(\epsilon, \delta)$-\discern \, and $r(\eta) < \delta$, $\sY$ is almost-everywhere $(\epsilon, r(\eta))$-\discern . Fix $\by_u = y_u$ such that $y_u$ is $(\epsilon, r(\eta))$-\discern . Then, $\forall y_v \in \sY$, $\rho^\prime(y_u, y_v) > 1 - r(\eta)$ implies that $d_{\sY}(y_u, y_v) \leq \epsilon$. Thus,
\begin{align*}
{\Pr}_{\by_v}(\rho^\prime(y_u, \by_v) > 1 - r(\eta)) & \leq {\Pr}_{\by_v}(d_{\sY}(y_u, \by_v \leq \epsilon) \\
& = 1 - {\Pr}_{\by_v}(d_{\sY}(y_u, \by_v) > \epsilon) \\
& \leq 1 - \kappa(\epsilon).
\end{align*}
Since the RHS does not depend on $y_u$, and $\by_u$ and $\by_v$ are independent, we can take the expectation with respect to $\by_u$ to establish the claim.

\textbf{Main Probability Bound: } Fix $\bOmega = \Omega$ such that $R_1 \leq |S| \leq R_2$.
\begin{align}
{\Pr}_{\by_v, \by_u}( {\max}^{(k)}_{v \in [S]}  \rho^\prime(\by_v, \by_u)  \leq & 1 - r(\eta) | \bOmega = \Omega) \nonumber \\
= & \sum_{l =0}^{k-1} {|S| \choose l} C_{\eta}^{|S| - l}(1-C_{\eta})^l  \nonumber \\
\leq &  k \max_{l \in \set{0, \ldots, k-1}} {|S| \choose l} C_{\eta}^{|S| - l} (1-C_{\eta})^l \nonumber \\
\leq &  k \max_{l \in [k-1]\cup \set{0}} {|S| \choose l} (1- \tau(\eta))^{|S| - l} (1-\kappa(\epsilon))^l\nonumber \\
\leq &  k \max_{l \in \set{0, \ldots, k-1}} (\frac{|S| e}{l})^l(1- \tau(\eta))^{|S| - l}  (1-\kappa(\epsilon))^l \label{good_neighbor_k_binomial}\\
\leq &  k \max_{l \in \set{0, \ldots, k-1}} \exp(l + l\log(\frac{|S|}{l}) - \tau(\eta)[|S| - l] - \kappa(\epsilon) l) \label{good_neighbor_k_exponential} \\
= &   k \max_{l \in \set{0, \ldots, k-1}} \exp([1  - \kappa(\epsilon) + \tau(\eta)]l + l \log(\frac{|S|}{l}) - \tau(\eta)|S|)) \nonumber \\
\leq &  k \exp([1  - \kappa(\epsilon) + \tau(\eta)]k + k\log(\frac{|S|}{k}) - \tau(\eta)|S|)) \label{large_enough_good_neighbor_k} \\
= &  \exp([1  - \kappa(\epsilon) + \tau(\eta)+ \log(|S|)]k - k \log(k)  - \tau(\eta)|S|)) \nonumber \\
\leq & \exp([1  - \kappa(\epsilon) + \tau(\eta)+ \log(R_1)]k - k \log(k)  - \tau(\eta)R_2)) \nonumber 
\end{align}
where line \eqref{good_neighbor_k_binomial} follows from the the inequality ${n \choose k} \leq (\frac{n e}{k})^k$, line \eqref{good_neighbor_k_exponential} follows from the inequality $(1-x) \leq \exp(-x)$, and line \eqref{large_enough_good_neighbor_k} follows since $|S| \geq k$ and $1 - \kappa(\epsilon) > 0$ by Lemma \ref{kappa_tau_lemma}. Finally, we can take the expectation with respect to $\bOmega = \Omega$ over the set $R_1 \leq |S| \leq R_2$ to conclude the result.
\end{proof}

Lemma \ref{R_concentration_discrete} is the analogoue of Lemma \ref{R_concentration} for the discrete case.
\begin{lemma}
\label{R_concentration_discrete}
Consider the discrete ratings setting. Let $u \neq v \in [n_2]$, $i \neq j \in [n_1]$, $\eta > 0$, $\beta \geq 2$, and $W^{i,j}_u(\beta)$ be defined as in Lemma \ref{sufficient_overlap}. Then,
\begin{align*}
\Pr(|R_{\bu, \bv} - \nrho(\by_u, \by_v)| > \frac{\eta}{4} | v \in W^{i,j}_u(\beta)) \leq 2 \exp(- \frac{\eta^2}{4} \left \lfloor\frac{\beta}{2} \right \rfloor ).
\end{align*}
\end{lemma}
\begin{proof}
Fix $\by_u = y_u$, $\by_v = y_v$, and $\bg_u = g_u, \bg_v = g_v$. Recall that if $I(u,v) \neq \emptyset$, then
\begin{align*}
R_{u,v} = \frac{1}{ |I(u,v)| } \sum_{(s,t) \in I(u,v)} \ind{ [h_u(\bx_s,y_u) - h_u(\bx_t, y_u)][h_v(\bx_s, y_v) - h_v(\bx_t, y_v)] \geq 0}.
\end{align*}
Since $I(u,v)$ consists of pairs of indices that do not overlap, conditioned on $\by_v = y_v$, $\by_u = y_u$, $\bg_u = g_u$, $\bg_v = g_v$ and any nonempty $I(u,v)$,  
\begin{align*}
\set{ \ind{(g_u(f(\bx_s,y_u)) - g_u(f(\bx_t, y_u))) (g_v(f(\bx_s, y_v)) - g_v(f(\bx_t, y_v))) \geq 0}:  (s,t) \in I(u,v)}
\end{align*}
is a set of independent random variables. Further, each has mean $\rho^\prime(y_u,y_v| \bg_u = g_u, \bg_v = g_v)$. Thus, by Chernoff's bound (Proposition \ref{hoeffding}),
\begin{align*}
& {\Pr}(|R_{u, v} - \rho^\prime(y_u, y_v |  \bg_u = g_u, \bg_v = g_v) | > \frac{\eta}{4} |   \by_u = y_u, \by_v = y_v, \bg_u = g_u, \bg_v = g_v , I(u,v)) \\
 & \qquad \leq  \exp(-\frac{\eta^2}{2} |I(u, v)|)
\end{align*}
When $v \in v \in W^{i,j}_u(\beta)$, $|I(u,v)| \geq \left \lfloor\frac{\beta}{2} \right \rfloor $. Since the above bound holds for all $\by_u, \by_v, \bg_u \bg_v $, it follows that
\begin{align*}
{\Pr}(|R_{\bu, \bv} - \rho^\prime(\by_u, \by_v)| > \frac{\eta}{4} | v \in W^{i,j}_u(\beta)) \leq 2 \exp(- \frac{\eta^2}{4} \left \lfloor\frac{\beta}{2} \right \rfloor ).
\end{align*}
\end{proof}

\begin{lemma}
\label{monotonic_function_lemma}
Let $\epsilon, \delta > 0$, $\frac{1}{2}>\alpha > \alpha^\prime > 0$, and $r$ be a positive nondecreasing function such that $r(\frac{\epsilon}{4}) \geq \delta$ and $r(\eta) < \frac{\delta}{2}$ for some $\eta >0$. Suppose that almost every $y \in \sY$ is $r$-\diverse \, and $(\frac{\epsilon}{4}, \delta)$-\discern . Fix $u \in [n_2]$, $i \neq j \in [n_1]$, and $k \leq \frac{(n_2 - 1) p^2}{2}$. As in the proof of Theorem \ref{pairwise_rank_levels_thm}, define
\begin{align*}
A & = \set{|W^{i,j}_u(\beta)| \in [\frac{(n_2 - 1) p^2}{2}, \frac{3(n_2 - 1) p^2}{2}]}, \\
B & = \set{{\max}^{(k)}_{v \in W^{i,j}_u(\beta)} \rho^\prime(\by_u,\by_v) \geq 1 - \frac{\delta}{2}}, \\
E & = \set{|f(\bx_i, \by_u) - f(\bx_j, \by_u)| > \epsilon} \\
M & = \set{ \exists v \in W^{i,j}_u(\beta) \text{ s.t. } \rho^\prime(\by_u, \by_v) \geq 1 - \frac{\delta}{2} \text{ and } \exists l \in [L-1] \text{ s.t. } \ba_{v,l} \in (f(\bx_j, \by_v), f(\bx_i, \by_v))}.
\end{align*}
Then,
\begin{align*}
\Pr(M^c | A, B, E) \leq \exp(-\gamma(\frac{\epsilon}{4}) k ).
\end{align*}
\end{lemma}

\begin{proof}
Fix $\set{\by_v = y_v}_{v \in [n_2]}$ $r$-\diverse \, and $(\frac{\epsilon}{4}, \delta)$-\discern , $\bOmega = \Omega$, and $\set{\bx_s = x_s}_{s \in [n_1]}$ such that $A \cap B \cap E$ holds. Let $R = \set{v \in [n_2] \setminus \set{u} : v \in W_u^{i,j}(\beta) \text{ and } \rho^\prime(y_u, y_v) \geq 1 - \frac{\delta}{2}}$. Events $A$ and $B$ imply that $|R| \geq k$. Since $y_u$ is $(\frac{\epsilon}{4}, \delta)$-\discern \, and for all $v \in R$, $\rho^\prime(y_u, y_v) \geq 1 - \frac{\delta}{2}$, it follows that for all $v \in R$, $y_v \in B_{\frac{\epsilon}{4}}(y_u)$. 

By $E$, $|f(x_i,y_u) - f(x_j,y_u)| > \epsilon$. Suppose that $f(x_i,y_u) > f(x_j,y_u) + \epsilon$ (the other case is similar). Then, by Lipschitzness of $f$, for all $v \in R$
\begin{align*}
f(x_j,y_v) \leq f(x_j, y_u) + \frac{\epsilon}{4} < f(x_i,y_u) - \frac{3}{4} \epsilon \leq f(x_i,y_v) - \frac{\epsilon}{2}.
\end{align*}
Thus, for all $v \in R$, $(f(x_j,y_v), f(x_i,y_v))$ is an open interval of length at least $\frac{\epsilon}{2}$. Fix $v^\prime \in [n_2] \setminus \set{u}$. Since $R$ is a finite set, the following is well-defined: 
\begin{align}
I \coloneq \argmin_{J \in \set{(f(x_j,y_v), f(x_i,y_v)) : v \in R}} {\Pr}_{\set{\ba_{v^\prime, l}}_{l \in [L-1]}}(\exists l \in [L-1] \text{ s.t. } a_{v^\prime, l} \in J). \label{smallest_measure_monotonic_functions_lemma}
\end{align}

Then,
\begin{align}
& {\Pr}_{\set{\ba_{v, l}}}( \forall  v \in R, \forall l \in [L-1], \, \ba_{v,l} \not \in (f(x_j, y_v), f(x_i, y_v)) | \set{\by_v = y_v}_{v \in [n_2]}, \bOmega = \Omega, \set{\bx_s = x_s}_{s \in [n_2]} ) \nonumber \\
& \qquad =   {\Pr}_{\set{\ba_{v, l}}}(\forall  v \in R, \forall l \in [L-1], \, \ba_{v,l} \not \in (f(x_j, y_v), f(x_i, y_v))) \label{independence_monotonic_functions_lemma} \\
& \qquad \leq  {\Pr}_{\set{\ba_{v, l}}}(\forall v \in R, \forall l \in [L-1] , \, \ba_{v,l} \not \in I) \label{apply_smallest_measure_monotonic_functions_lemma} \\
& \qquad =  {\Pr}_{\set{\ba_{v^\prime, l}}_{l \in [L-1]}}(\forall l \in [L-1] , \, \ba_{v^\prime,l} \not \in I)^k \label{independent_functions_monotonic_functions_lemma} \\
& \qquad =  [1 - {\Pr}_{\set{\ba_{v^\prime, l}}_{l \in [L-1]}}(\exists l \in [L-1] \text{ s.t. }  \ba_{v^\prime,l}  \in I)]^k \nonumber \\
& \qquad \leq  (1 - \gamma(\frac{\epsilon}{4}))^k \label{gamma_line_monotonic_functions_lemma} \\
& \qquad \leq  \exp(-\gamma(\frac{\epsilon}{4}) k ). \label{exp_bound_monotonic_functions_lemma} 
\end{align}

Line \eqref{independence_monotonic_functions_lemma} follows from the independence of $\set{\by_v}_{v \in [n_2]}$, $\bOmega$, and $\set{\bx_s}_{s \in [n_1]}$ from $\set{\ba_{v,l}}_{v \in [n_2], l \in [L-1]}$. Line \eqref{apply_smallest_measure_monotonic_functions_lemma} follows from the definition of $I$ in line \eqref{smallest_measure_monotonic_functions_lemma} and because the monotonic functions $\set{\bg_v}_{v \in [n_2]}$ are identically distributed. Line \eqref{independent_functions_monotonic_functions_lemma} follows since $\set{\bg_v}_{v \in R}$ are i.i.d., line \eqref{gamma_line_monotonic_functions_lemma} follows from the definition of $\gamma$, and line \eqref{exp_bound_monotonic_functions_lemma} follows from the inequality $1-x \leq \exp(-x)$. Note that since $\sP_{\sG}$ is diverse by hypothesis, by Lemma \ref{kappa_tau_lemma}, $\gamma(\frac{\epsilon}{4}) > 0$.

Since $\set{\by_v}_{v \in [n_2]}$, $\bOmega \cup \set{\bx_s}_{s \in [n_1]}$, and $\set{\ba_{v,l}}_{v \in [n_2], l \in [L-1]}$ are independent and almost every $y \in \sY$ is $r$-\diverse \, and $(\frac{\epsilon}{4}, \delta)$-\discern , taking the expectation of line \eqref{exp_bound_monotonic_functions_lemma} with respect to $\set{\by_v}_{v \in [n_2]}$, $\bOmega$, and $\set{\bx_s}_{s \in [n_1]}$  over $A \cap B \cap E$ finishes the proof.

\end{proof}

Lemma \ref{error_bound_discrete} gives a bound on the probability of $D_{u,i,j}^\epsilon$ conditional on $A \cap B \cap C \cap E \cap M$ (defined in the proof of Theorem \ref{pairwise_rank_levels_thm}).

\begin{lemma}
\label{error_bound_discrete}
Under the setting described in Theorem \ref{pairwise_rank_levels_thm}, let $u \in [n_2]$ and $i \neq j \in [n_1]$. Then,
\begin{align*}
\Pr(D_{u,i,j}^\epsilon| A, B, C, E, M) = 0.
\end{align*}
\end{lemma}

\begin{proof}
Define the sets
\begin{align*}
E_1 & = \set{f(\bx_i, \by_u) + \epsilon < f(\bx_j, \by_u)} \\
E_2 & = \set{f(\bx_i, \by_u)  > f(\bx_j, \by_u)+ \epsilon}.
\end{align*}
Then, by the union bound and the law of total probability,
\begin{align*}
\Pr(D_{u,i,j}^\epsilon | A,B,C,E,M) & \leq \Pr(\text{PR}(u,i,j,\beta,k) = 1  \cap E_1 | A,B,C,E,M) \\
& +  \Pr(\text{PR}(u,i,j,\beta,k) = 0  \cap E_2 | A,B,C,E,M) \\
& \leq  \Pr(\text{PR}(u,i,j,\beta,k) = 1  | A,B,C, E_1, M) \\
& +   \Pr(\text{PR}(u,i,j,\beta,k) = 0 | A,B,C, E_2, M).
\end{align*}
The argument for bounding each of these terms is similar, so we only bound $ \Pr(\text{PR}(u,i,j,\beta,k) = 1  | A,B,C, E_1, M) $. 

Fix $\set{\by_v = y_v}_{v \in [n_2]}$ $r$-\diverse \, and $(\frac{\epsilon}{4}, \delta)$-\discern , $\set{\bx_s = x_s}_{s \in [n_1]}$, $\bOmega = \Omega$, and $\set{\ba_{v,l} = a_{v,l}}_{v \in [n_2], l \in [L-1]}$ such that $A \cap B \cap C \cap E_1 \cap M$ occurs . We claim that the set $V$ in Pairwise-Rank consists of $v_1, \ldots, v_k \in W^{i,j}_u(\beta)$ such that for all $l \in [k]$, $y_{v_l} \in B_{\frac{\epsilon}{4}}(y_u)$. The event $B$ implies that there are $v_1, \ldots, v_k$ such that for all $l \in [k]$, $\rho^\prime(y_u, y_{v_l}) \geq 1 - \frac{\delta}{2}$. Then, since $y_u$ is $(\frac{\epsilon}{4}, \delta)$-\discern , it follows that $y_{v_1}, \ldots, y_{v_k} \in B_{\frac{\epsilon}{4}}(y_u)$. Suppose that $w \in W_u^{i,j}(\beta)$ such that $y_w \in B_{\frac{\epsilon}{4}}(y_u)^c$. Then, since $y_u$ is $(\frac{\epsilon}{4}, \delta)$-\discern , it follows that that $\rho^\prime(y_u, y_w) < 1 - \delta$. Then, for all $l \in [k]$,
\begin{align}
R_{w,u} & \leq \rho^\prime(y_w,y_u) + \frac{\delta}{4} \label{apply_C_1_error_bound_discrete} \\
&  < 1 - \frac{3}{4} \delta \nonumber \\
& \leq \rho^\prime(y_u, y_{v_l}) - \frac{\delta}{4} \nonumber  \\
& \leq R_{u,v_l} \label{apply_C_2_error_bound_discrete}
\end{align}
where lines \eqref{apply_C_1_error_bound_discrete} and \eqref{apply_C_2_error_bound_discrete} follow by event $C$  and $v_l,w \in W_u^{i,j}(\beta)$. Thus, Pairwise-Rank selects $v_1, \ldots, v_k \in W_u^{i,j}(\beta)$ such that for all $l \in [k]$, $y_{v_l} \in B_{\frac{\epsilon}{4}}(y_u)$. Thus, the claim follows. 

Event $E_1$ implies that $f(x_i, y_u) + \epsilon < f(x_j, y_u)$. Fix $l \in [k]$. Then, by the Lipschitzness of $f$,
\begin{align*}
f(x_i, y_{v_l}) \leq f(x_i, y_u) + \frac{\epsilon}{4} < f(x_j, y_u) - \frac{3 \epsilon}{4} \leq f(x_j, y_{v_l}) - \frac{\epsilon}{2}.
\end{align*}
Hence, $\forall l \in [k]$, $f(x_i, y_{v_l}) + \frac{\epsilon}{2} < f(x_j, y_{v_l})$ and $h_{v_l}(x_i, y_{v_l}) \leq h_{v_l}(x_j, y_{v_l})$. Then, event $M$ implies that there is some $l \in [k]$ such that $h_{v_l}(x_i, y_{v_l}) < h_{v_l}(x_j, y_{v_l})$. Thus, the majority vote outputs the correct result. Thus, 
\begin{align}
\Pr(\text{PR}(u,i,j,\beta,k) = 1  | & A,B,C, E_1, M, \nonumber \\ 
& \set{\by_v = y_v}_{v \in [n_2]}, \set{\bx_s = x_s}_{s \in [n_1]}, \nonumber \\
& \bOmega = \Omega, \set{\ba_{v,l} = a_{v,l}}_{v \in [n_2], l \in [L-1]}) = 0. \label{last_line_error_bound_discrete}
\end{align}
Since line \eqref{last_line_error_bound_discrete} holds for all  $\set{y_v}_{v \in [n_2]}$ $r$-\diverse \, and $(\frac{\epsilon}{4}, \delta)$-\discern , $\set{a_{v,l}}_{v \in [n_2], l \in [L-1]}$, $\set{x_s}_{s \in [n_1]}$, $ \Omega$ conditioned on the set the set $A \cap B \cap C \cap E_1 \cap M$ and almost every $y \in \sY$ is $r$-\diverse \, and $(\frac{\epsilon}{4}, \delta)$-\discern , the result follows.
\end{proof}

\section{Proofs for Section \ref{necessary_sufficient_section}}
\label{necessary_sufficient_section_app}

\begin{proof}[Proof of Theorem \ref{sufficient_condition_thm_continuous}]
By compactness of $\sY$, there exists a finite subcover $\set{C_1, \ldots, C_n}$ of $\sY$ where each open ball $C_i$ has diameter $\frac{\epsilon}{2}$. Since by assumption, for all $r >0$ and $y \in \sY$, $\sP_{\sY}(B_r(y)) > 0$, we have that $\sP_{\sY}(C_i) > 0$ for all $i = 1, \ldots, n$. Let $Q_{n_2}$ denote the event that for every $l \in [n]$ and $i,j \in [n_1]$, there exists $u \in [n_2]$ such that $\by_u \in C_l$ and we observe $(i,u) \in \Omega$ and $(j,u) \in \bOmega$. Since $p > 0$, as $n_2 \longrightarrow \infty$, $\Pr(Q_{n_2}) \longrightarrow 1$.

Let $\set{\bx_i = x_i}_{i \in [n_1]}$, $\set{\by_u = y_u}_{u \in [n_2]}$, and $\bOmega = \Omega$ such that $Q_{n_2}$ occurs. Let $\sigma \in \sS^{n_1 \times n_2}$ be an $\frac{\epsilon}{2}$-\local \, minimizer of $\edis(\cdot, H)$ over the sample. Towards a contradiction, suppose there exists $y_u$ and $i \neq j \in [n_1]$ such that $\sigma(i, u) < \sigma(j, u)$, $h_u(x_i, y_u) > h_u(x_j, y_u)$, and $f(x_i,y_u) > f(x_j,y_v) + \epsilon$. Without loss of generality, suppose that $y_u \in C_1$. 

Since $Q_{n_2}$ occurs by assumption, there exists $v \in [n_2]$ such that $y_v \in C_1$ and $(i,v), (j,v) \in \Omega$. Since $\sigma$ is an $\frac{\epsilon}{2}$-\local \, \multirank \, and the diameter of $C_1$ is $\frac{\epsilon}{2}$, $\sigma$ gives the same ranking to $y_u$ and $y_v$. Then, since $\sigma(i, u) < \sigma(j, u)$, it follows that $\sigma(i, v) < \sigma(j, v)$. By Lipschitzness of $f$,
\begin{align}
f(x_i, y_v) \geq f(x_i,y_u) - \frac{\epsilon}{2} > f(x_j,y_u) + \frac{\epsilon}{2} \geq f(x_j,y_v). \label{suff_condition_contradiction}
\end{align} 
Since $g_v$ is strictly increasing, line \eqref{suff_condition_contradiction} implies that $h_v(x_i,y_v) > h_v(x_j,y_v)$. Thus, $\sigma$ is not a minimizer of $\edis(\cdot, H)$--a contradiction. Thus, $\forall u \in [n_2]$ and $i \neq j \in [n_1]$ if $\sigma(i, u) < \sigma(j, u)$ and $h_u(x_i, y_u) > h_u(x_j, y_u)$, then $f(x_i, y_u) \leq f(x_j, y_u) + \epsilon$, implying that ${\dis}_\epsilon(\sigma, H) = 0$.
\end{proof}


\begin{proof}[Proof of Theorem \ref{sufficient_condition_thm_discrete}]
Fix $\set{\bx_i = x_i}_{i \in [n_1]}$. By compactness of $\sY$, there exists a finite subcover $\set{C_1, \ldots, C_n}$ of $\sY$ where each open ball $C_i$ has diameter $\frac{\epsilon}{8}$. For every $l \in [n]$, fix $z_l \in C_l$ and define $P_l = \set{(i,j) : f(x_i, z_l) > f(x_j, z_l) + \frac{\epsilon}{2}}$. 

Fix $l \in [n]$ and $(i,j) \in P_l$. Let $Q^{l,i,j}_{n_2}$ denote the event that there exists $\by_u \in C_l$ with $(i,u),(j,u) \in \bOmega$ and $\ba_{u,q} \in (f(x_j, \by_u), f(x_i,\by_u))$ for some $q \in [L-1]$. Further, define
\begin{align*}
Q_{n_2} = \cap_{l \in [n], (i,j) \in P_l} Q^{l,i,j}_{n_2}.
\end{align*}
Observe that by the Lipschitzness of $f$, for every $z \in C_l$, if $(i,j) \in P_l$, then $f(x_i, z) > f(x_j, z) + \frac{\epsilon}{4}$. Since $n$ is fixed and finite, $|P_l|$ is fixed and finite, and the probability of observing a rating, $p$, is fixed, there exists a positive constant $C > 0$ such that $\Pr_{\by_u, \bOmega}(Q^{l,i,j}_{n_2} \, | \, \set{\bx_s = x_s}_{s \in [n_1]}) \geq C$. Thus, $\Pr(Q^{l,i,j}_{n_2} \, | \, \set{\bx_s = x_s}_{s \in [n_1]}) \longrightarrow 1$ as $n_2 \longrightarrow \infty$. Then, by the union bound,
\begin{align*}
\lim_{n_2 \longrightarrow \infty} {\Pr}_{\by_u, \bOmega}([Q_{n_2}]^c \, | \, \set{\bx_s = x_s}_{s \in [n_1]}) & \leq \lim_{n_2 \longrightarrow \infty} n {\ n_1 \choose 2} {\Pr}_{\by_u, \bOmega}([Q^{l,i,j}_{n_2}]^c \, | \, \set{\bx_s = x_s}_{s \in [n_1]}) \\
& = 0.
\end{align*}

Since $\bbE[ \ind{Q_{n_2}} | \set{\bx_i}_{i \in [n_1]}] \leq 1$, by the dominated convergence theorem,
\begin{align*}
\lim_{n_2 \longrightarrow \infty} \Pr(Q_{n_2}) & = \lim_{n_2 \longrightarrow \infty}  {\bbE}_{\set{\bx_i}} \bbE[\ind{Q_{n_2}} | \set{\bx_i }_{i \in [n_1]}] \\
& = {\bbE}_{\set{\bx_i}} \lim_{n_2 \longrightarrow \infty} \bbE[\ind{Q_{n_2}} | \set{\bx_i}_{i \in [n_1]}]  \\
& = 1
\end{align*}

Now, condition on $\set{\bx_i = x_i}_{i \in [n_1]}, \set{\by_u = y_u}_{u \in [n_2]}, \bOmega = \Omega, \set{\ba_{u,l} = a_{u,l}}_{u \in [n_2], l \in [L-1]}$ such that $Q_{n_2}$ happens. Let $\sigma \in \sS^{n_1 \times n_2}$ be an $\frac{\epsilon}{8}$-\local \, minimizer of $\edis(\cdot, H)$. Towards a contradiction, suppose there exists $y_u$ and $i \neq j \in [n_1]$ such that $\sigma(i, u) < \sigma(j, u)$, $h_u(x_i, y_u) > h_u(x_j, y_u)$, and $f(x_i,y_u) > f(x_j,y_v) + \epsilon$. Without loss of generality, suppose that $y_u \in C_1$. We have that $(i,j) \in P_1$ since
\begin{align*}
f(x_i,z_1) & \geq f(x_i,y_u) - \frac{\epsilon}{8} \\
& \geq f(x_j,y_u) + \frac{7}{8} \epsilon \\
& \geq f(x_j,z_1) + \frac{3}{4} \epsilon.
\end{align*}

Therefore, the event $Q_{n_2}$ implies that there exists $y_v \in C_1$ such that $(i,v), (j,v) \in \Omega$ and there exists $a_{v,q} \in (f(x_j,y_v), f(x_i,y_v))$. By the Lipschitzness of $f$, $f(x_j,y_v) < f(x_i,y_v)$, so that $h(x_j,y_v) < h(x_i,y_v)$. Since $\sigma$ is $\frac{\epsilon}{8}$-\local , $\sigma(i, v) < \sigma(j, v)$. But, then $\sigma$ is not a minimizer of $\edis(\cdot, H)$ over the sample--a contradiction.  Thus, $\forall u \in [n_2]$ and $i \neq j \in [n_1]$ if $\sigma(i, u) < \sigma(j, u)$ and $h_u(x_i, y_u) > h_u(x_j, y_u)$, then $f(x_i, y_u) \leq f(x_j, y_u) + \epsilon$, implying that ${\dis}_\epsilon(\sigma, H) = 0$.
\end{proof}

\begin{proof}[Proof of Theorem \ref{necessary_condition_thm}]
Let $\bx_1 = x_1, \ldots, \bx_{n_1} = x_{n_1}, \by_1 = y_1, \ldots, \by_{n_2} = y_{n_2}$. Towards a contradiction, suppose that $\sigma$ is not an $\epsilon$-\local \, \multirank \, over $T$. Then, there exists $i,j \in [n_1]$ and $u,v \in [n_2]$ such that $(i,j,u), (i,j,v) \in T$ and
\begin{align}
& d_{\sY}(y_u, y_v)  \leq \epsilon, \label{2_nec_suff_distinance_line} \\
& \sigma(j,u)  < \sigma(i,u), \label{2_nec_suff_line_1} \\
& \sigma(j,v) > \sigma(i,v). \label{2_nec_suff_line_2}
\end{align}
Further, by definition of $T$, 
\begin{align}
|f(x_j, y_u) - f(x_i, y_u)| & > \epsilon \label{2_nec_suff_line_3} \\
|f(x_i, y_v) -f(x_j, y_v)| & > \epsilon. \label{2_nec_suff_line_4} \\
h(x_i,y_u)  \neq h(x_j,y_u) & \label{2_nec_suff_line_5} \\
h(x_i,y_v)  \neq h(x_j,y_v) & \label{2_nec_suff_line_6} 
\end{align}
Since  ${\dis}_\epsilon(\sigma, H) = 0$ by hypothesis, and by inequalities \eqref{2_nec_suff_line_1}, \eqref{2_nec_suff_line_2}, \eqref{2_nec_suff_line_3}, \eqref{2_nec_suff_line_4}, \eqref{2_nec_suff_line_5}, and \eqref{2_nec_suff_line_6} it follows that $h(x_j,y_u) < h(x_i,y_u)$ and $h(x_i,y_v) < h(x_j,y_v)$. Thus, by monotonicity of $g_u, g_v$,
\begin{align*}
\epsilon + f(x_j, y_u) & < f(x_i, y_u), \\
\epsilon + f(x_i, y_v) & < f(x_j, y_v).
\end{align*}
Then,
\begin{align*}
f(x_i, y_u) - f(x_i, y_v) & = f(x_i, y_u) - f(x_j, y_u) + f(x_j, y_u) - f(x_j, y_v) + f(x_j, y_v) - f(x_i, y_v) \\
& > 2 \epsilon + f(x_j, y_u) - f(x_j, y_v).
\end{align*}
Then, rearranging the above equation and applying the Lipschitzness of $f$, we have that
\begin{align*}
2 \epsilon < f(x_j, y_v) - f(x_j, y_u) + f(x_i, y_u) - f(x_i, y_v) \leq 2 d_{\sY}(y_v,y_u),
\end{align*}
which contradicts inequality \eqref{2_nec_suff_distinance_line}.
\end{proof}

\section{Proof of Proposition \ref{example_prop} and other Results}
\label{example_sec}

In the following proposition, we give a simple illustrative example of a $1$-Lipschitz function that is  $(\epsilon, \delta)$-\discern \, and $r$-\diverse . 

\begin{prop}
\label{example}
Let $\sX = [0,1]$, $\sY = [0,1]$, $\sP_\sX$ be the Lebesgue measure over $\sX$, and $\sP_{\sY}$ be the Lebesgue measure over $\sY$. Suppose that for all $u \in [n_2]$, $g_u$ is strictly increasing. Consider the function
\begin{align*}
   f(x,y) = \left\{
     \begin{array}{lr}
       x & : x \in [0, y]\\
      y-x & : x \in (y, 1]
     \end{array}
   \right.
\end{align*}
Then, for all $1 > \epsilon > 0$, every $y \in \sY$ is $(\epsilon, \epsilon^2)$-\discern  . Further, there exists a positive nondecreasing $r$ such that $\lim_{r \longrightarrow 0} r(z) = 0$ and every $y \in \sY$ is $r$-\diverse .
\end{prop}

\begin{proof}
Let $\epsilon \in (0,1)$ and suppose that $|y_1 - y_2| = \epsilon$. Without loss of generality, suppose that $y_1 < y_2$. Then, when $x_1 < x_2 \in (y_1, y_1 + \epsilon)$, $f(x_1, y_1) > f(x_2, y_1)$ and $f(x_1, y_2) < f(x_2, y_2)$. Since $g_u$ is strictly increasing, $h_1(x_1, y_1) > h_1(x_2, y_1)$ and $h_2(x_1, y_2) < h_2(x_2, y_2)$. Since $\sP_{\sX} \times \sP_{\sX}((y_1, y_1 + \epsilon) \times (y_1, y_1 + \epsilon) ) = \epsilon^2$, it follows that $\rho(y_1, y_2) < 1 - \epsilon^{2}$. 

Clearly, there exists a positive nondecreasing $r$ such that $\lim_{r \longrightarrow 0} r(z) = 0$ and every $y \in \sY$ is $r$-\diverse .
\end{proof}

This example can easily be generalized to $f(x,y) = \norm{x - y}_2$. The following proposition shows that by adding a dimension, the model $f(\bx, \by) = \bx^t \by$ with $\bx, \by \in \bbR^d$  is a special case of the model $f(\tilde{\bx},\tilde{\by}) = \norm{\tilde{\bx} - \tilde{\by}}_2$ with $\tilde{\bx}, \tilde{\by} \in \bbR^{d+1}$. A similar construction in the other direction exists.

\begin{prop}
Let $\bx_1, \ldots, \bx_{n_1} \in \bbR^{d}$ and $\by_1,\ldots, \by_{n_2} \in \bbR^d$. There exist $\tilde{\bx}_1, \ldots, \tilde{\bx}_{n_1} \in \bbR^{d+1}$ and $\tilde{\by}_1, \ldots, \tilde{\by}_{n_2} \in \bbR^{d+1}$ such that $\forall u \in [n_2]$ and $\forall i \neq j \in [n_1]$, $\bx_i^t \by_u > \bx_j^t \by_u$ if and only if $\norm{\tilde{\bx}_i - \tilde{\by}_u}_2 > \norm{\tilde{\bx}_j - \tilde{\by}_u}_2$. 
\end{prop}

\begin{proof}
Let $B = \max_{i \in [n_1]} \norm{\bx_i}_2$. For all $i \in [n_1]$, there exists $\gamma_i \geq 0$ such that $\tilde{\bx}_i \coloneq (\bx_i^t, \gamma_i)^t$ and $\norm{\tilde{\bx}_i}_2 = B$ (by continuity and monotonicity of $\norm{\cdot}_2$). For all $u \in [n_2]$, define $\tilde{\by}_u  = (- \by^t_u, 0)^t$.

Fix $u \in [n_2]$ and $i \neq j \in [n_1]$. Then, 
\begin{align*}
\norm{\tilde{\bx}_i - \tilde{\by}_u}^2_2 - \norm{\tilde{\bx}_j - \tilde{\by}_u}^2_2 & = \norm{\tilde{\bx}_i}_2^2 + \norm{\tilde{\by}_u}_2^2 - 2 \tilde{\bx}_i^t \tilde{\by}_u - (\norm{\tilde{\bx}_j}_2^2 + \norm{\tilde{\by}_u}_2^2 - 2 \tilde{\bx}_j^t \tilde{\by}_u) \\
& = - 2 \tilde{\bx}_i^t \tilde{\by}_u + 2 \tilde{\bx}_j^t \tilde{\by}_u \\
& = \bx_i^t \by_u - \bx_j^t \by_u.
\end{align*}
The result follows.
\end{proof}

\begin{proof}[Proof of Proposition \ref{example_prop}]

\begin{enumerate}

\item Consider a fixed $y \in \sY$. Fix $\bx_2 = x_2 \in \sX$. Then,
\begin{align*}
{\Pr}_{\bx_1}(| \norm{\bx_1 - y}_2 - \norm{\bx_2 - y}_2| \leq 2 \epsilon) & \leq {\Pr}_{\bx_1}( \bx_1 \in B_{\norm{\bx_2 - y} + 2 \epsilon}(y) \setminus  B_{\norm{\bx_2 - y} - 2 \epsilon}(y) ) \\
& \leq \sup_{z \in [0,2]} {\sP}_{\sX}(B_z(y) \setminus B_{z - 4 \epsilon}(y)) \\
& = r(\epsilon)
\end{align*}
Taking the expectation with respect to $\bx_2$ establishes the first part of this result. 

Fix $y_u \in \sY$ and $\epsilon > 0$ and set $\delta = 2\sP_{\sX}(B_{\frac{\epsilon}{2}}(y_u))^2$. Fix $y_v \in B_{\epsilon}(y_u)^c \cap \sY$. If $\bx_1 = x_1 \in  B_{\frac{\epsilon}{2}}(y_u)$ and $\bx_2 = x_2 \in B_{\frac{\epsilon}{2}}(y_v)$, then 
\begin{align*}
[f(x_1,y_u) - f(x_2, y_u)][f(x_1,y_v) - f(x_2,y_v)] < 0.
\end{align*}
A similar argument applies to the case $\bx_1 = x_1 \in  B_{\frac{\epsilon}{2}}(y_v)$ and $\bx_2 = x_2 \in B_{\frac{\epsilon}{2}}(y_u)$. Thus, since by hypothesis, $g_u$ is strictly increasing for all $u \in [n_2]$,
\begin{align*}
\rho(y_u,y_v) < 1 - \delta.
\end{align*}

\item Both results follow immediately.

\end{enumerate}

\end{proof}

\section{Useful Bounds}
\label{useful_bounds_section}

\begin{prop}[Chernoff-Hoeffding's Bound]
\label{hoeffding}
Let $X_1, \ldots, X_n$ be independent random variables with $X_i \in [a_i, b_i]$. Let $\bar{X} = \frac{1}{n} \sum_{i=1}^n X_i$. Then,
\begin{align*}
\Pr(|\bar{X} - \bbE[\bar{X}]| \geq t) \leq 2 \exp(-\frac{2n^2 t^2}{\sum_{i=1}^n (b_i-a_i)^2}).
\end{align*}
\end{prop}

\begin{prop}[Chernoff's multiplicative bound]
\label{chernoff_mult}
Let $X_1, \ldots, X_n$ be independent random variables with values in $[0,1]$. Let $X = \sum_{i=1}^n X_i$. Then, for any $\epsilon > 0$,
\begin{align*}
\Pr(X > (1+ \epsilon) \bbE[X]) & < \exp(-\frac{ \epsilon^2 \bbE[X]}{3}), \\
\Pr(X < (1- \epsilon) \bbE[X]) & < \exp(-\frac{ \epsilon^2 \bbE[X]}{2}).
\end{align*}
\end{prop}

\subsubsection*{Acknowledgements}

This work was supported in part by NSF grant 1422157.

\bibliography{references}

\end{document}

%% file: results_avg_standard.txt
\begin{tabular}{llllll}
\hline
 Dataset & Method   & Kendall Tau     & Spearman Rho    & NDCG@5          & Precision@5     \\
 \hline
            &  MRW      & 0.3156 (0.0021) & 0.4012 (0.0029) & 0.7104 (0.0010) & 0.4492 (0.0018) \\
            &  MR       & 0.3105 (0.0021) & 0.3963 (0.0030) & 0.7063 (0.0038) & 0.4457 (0.0044) \\
            &  LA       & 0.3271 (0.0018) & 0.4153 (0.0022) & 0.7136 (0.0026) & 0.4570 (0.0041) \\
            &  AltSVM   & 0.3271 (0.0007) & 0.4173 (0.0008) & 0.7022 (0.0015) & 0.4365 (0.0036) \\
Netflix     &  RMC      & 0.3288 (0.0017) & 0.4178 (0.0020) & 0.7204 (0.0006) & 0.4581 (0.0048) \\
\hline
            &   MRW      & 0.3933 (0.0010) & 0.5009 (0.0013) & 0.7769 (0.0066) & 0.6083 (0.0096) \\
            &   MR       & 0.3924 (0.0011) & 0.4999 (0.0013) & 0.7735 (0.0061) & 0.6021 (0.0063) \\
            &   LA       & 0.3993 (0.0009) & 0.5075 (0.0012) & 0.7767 (0.0058) & 0.6071 (0.0080) \\
            &   AltSVM   & 0.4099 (0.0008) & 0.5219 (0.0010) & 0.8002 (0.0042) & 0.6417 (0.0067) \\
Movielens   &   RMC      & 0.4041 (0.0004) & 0.5139 (0.0006) & 0.8068 (0.0030) & 0.6485 (0.0029) \\
\hline
\end{tabular}

%% file: results_avg_transformed.txt
\begin{tabular}{llllll}
\hline
 Dataset & Method   & Kendall Tau     & Spearman Rho    & NDCG@5          & Precision@5     \\
\hline
        & MRW      & 0.2736 (0.0017) & 0.3327 (0.0021) & 0.8063 (0.0031) & 0.7849 (0.0034) \\
        & MR       & 0.2677 (0.0019) & 0.3255 (0.0023) & 0.7980 (0.0034) & 0.7764 (0.0008) \\
        & LA       & 0.2786 (0.0022) & 0.3387 (0.0027) & 0.8024 (0.0024) & 0.7843 (0.0018) \\
        & AltSVM   & 0.2743 (0.0015) & 0.3335 (0.0018) & 0.7949 (0.0023) & 0.7768 (0.0023) \\
Netflix & RMC      & 0.2856 (0.0017) & 0.3473 (0.0021) & 0.8052 (0.0038) & 0.7861 (0.0032) \\
\hline
            & MRW      & 0.3347 (0.0015) & 0.4090 (0.0018) & 0.8903 (0.0059) & 0.8810 (0.0059) \\
            & MR       & 0.3343 (0.0017) & 0.4085 (0.0021) & 0.8879 (0.0052) & 0.8792 (0.0061) \\
            & LA       & 0.3395 (0.0017) & 0.4149 (0.0020) & 0.8908 (0.0085) & 0.8845 (0.0078) \\
            & AltSVM   & 0.3451 (0.0016) & 0.4217 (0.0020) & 0.9070 (0.0056) & 0.8982 (0.0056) \\
Movielens   & RMC      & 0.3504 (0.0014) & 0.4281 (0.0017) & 0.9140 (0.0026) & 0.9051 (0.0032) \\
\hline
\end{tabular}

%% file: results_avg_monotonic.txt
\begin{tabular}{llllll}
\hline
Dataset & Method   & Kendall Tau     & Spearman Rho    & NDCG@5          & Precision@5     \\
\hline
Netflix & LA       & 0.1798 (0.0034) & 0.2300 (0.0040) & 0.5962 (0.0022) & 0.3322 (0.0053) \\
\hline
MovieLens & LA       & 0.2404 (0.0098) & 0.3092 (0.0123) & 0.6543 (0.0138) & 0.4435 (0.0163) \\
\hline
\end{tabular}